\documentclass{article}
\usepackage{iclr2027_conference,times}
\usepackage{amsmath,amssymb,amsthm,mathtools}
\usepackage{booktabs,graphicx,float,xcolor}
\usepackage{hyperref,url}
\usepackage[breakable,skins]{tcolorbox}
\usepackage{wrapfig}
\newcommand{\Needspace}[1]{\ifdim\dimexpr\pagegoal-\pagetotal\relax<#1\relax\newpage\fi}
\renewcommand{\eqref}[1]{\textup{(\ref{#1})}}
\newtheorem{theorem}{Theorem}
\newtheorem{corollary}[theorem]{Corollary}
\newtheorem{lemma}{Lemma}
\newtheorem{assumption}{Assumption}
\tcbset{statementbox/.style={
    enhanced jigsaw, breakable, lines before break=6,
    colback=white, colframe=black!35,
    boxrule=0.6pt, sharp corners,
    boxsep=0pt, left=5pt, right=5pt, top=3pt, bottom=3pt,
    before skip=6pt, after skip=6pt,
    toprule at break=0pt, bottomrule at break=0pt, pad at break=3pt,
    colbacktitle=white, coltitle=black, fonttitle=\normalfont\bfseries,
    titlerule=0pt, toptitle=3pt, bottomtitle=2pt
}}
\tcolorboxenvironment{theorem}{statementbox,title after break={Theorem~\thetheorem\ (continued)}}
\tcolorboxenvironment{lemma}{statementbox,title after break={Lemma~\thelemma\ (continued)}}
\tcolorboxenvironment{corollary}{statementbox,title after break={Corollary~\thecorollary\ (continued)}}
\tcolorboxenvironment{assumption}{statementbox,title after break={Assumption~\theassumption\ (continued)}}

\newtcolorbox{summarybox}{colback=blue!3!white,colframe=blue!20!gray,
  boxrule=0.4pt,arc=2mm,boxsep=0pt,left=6pt,right=6pt,top=3pt,bottom=3pt,
  before skip=3pt,after skip=4pt,fontupper=\itshape,
  before upper={\setlength{\parskip}{0pt}}}

\newcommand{\mbb}[1]{\mathbb{#1}}
\newcommand{\mcl}[1]{\mathcal{#1}}
\newcommand{\N}{\mathcal{N}}

\title{The Noise-Transmission Burden of Diffusion Network: Why Image Prediction Can Outperform Noise Prediction}
\author{Gyeonghun Kang\thanks{Equal contribution.}\quad
Pengxi Liu\footnotemark[1]\quad
Bingshuo Qian\footnotemark[1]\quad
Xiang Cheng\\
{\normalfont Duke University, Durham, NC, USA}\\
{\normalfont\small\texttt{\{gyeonghun.kang,pengxi.liu,bingshuo.qian,xiang.cheng\}@duke.edu}}}
\iclrfinalcopy
\hypersetup{
  pdftitle={The Noise-Transmission Burden of Diffusion Network: Why Image Prediction Can Outperform Noise Prediction},
  pdfauthor={Gyeonghun Kang, Pengxi Liu, Bingshuo Qian, Xiang Cheng}
}
\begin{document}
% The contents list is confined to the appendix.
\addtocontents{toc}{\protect\setcounter{tocdepth}{-1}}
\maketitle
\fancyhead[L]{Preprint}
% Keep figure spacing compact without stretching gaps to fill the page.
\raggedbottom
\setlength{\intextsep}{8pt plus 1pt minus 1pt}
\setlength{\textfloatsep}{10pt plus 2pt minus 2pt}
\brokenpenalty=10000
\begin{abstract}
Clean-image prediction ($x$-prediction) and noise prediction yield equivalent population-optimal denoisers, yet can perform very differently in finite diffusion networks.
For noisy input $Y_t$ with noise standard deviation $\sigma_t$, a noise predictor $f_{\varepsilon,t}$ induces the clean-image estimate $\widehat X_1=Y_t-\sigma_t f_{\varepsilon,t}(Y_t)$.
This reveals a \textit{noise-transmission burden}: the predictor must preserve information needed to cancel the input corruption.
Under Gaussian corruption of $d$-dimensional data, a linear input bottleneck of rank $r<d$ imposes a denoising-error floor $\sigma_t^2(d-r)$ even with unrestricted downstream computation, while the corresponding optimal risk for image prediction remains bounded.
EDM attenuates this bottleneck contribution at high noise; exact preconditioning cancels it at every noise level.
Without an input bottleneck, assuming exact affine recovery of the noisy patch from the model representation, we characterize how the covariance spectrum of corrections beyond linear denoising determines the width required above the patch dimension.
ImageNet experiments with JiT show that measured denoising losses reflect the predicted bottleneck floor and its attenuation or cancellation through preconditioning.
Further widening yields larger relative denoising gains at low noise, where corrections beyond linear denoising spread across more directions, consistent with our spectral characterization of the required model width.
Together, these results provide a denoising-based account of practical generation-quality gaps across prediction targets, preconditioning choices, and model widths.
\end{abstract}

\section{Introduction}
\label{sec:introduction}
Why should predicting the clean image be easier than predicting the noise added to it? 
At the population optimum, image, noise, and velocity predictors are related by simple identities and yield the same denoiser. 
Yet their performance can differ substantially in finite networks. 
JiT demonstrates this gap in pixel-space diffusion Transformers: image prediction succeeds with coarse patches and narrow input embeddings in settings where noise or velocity prediction performs poorly \citep{li2025jit}. 
Understanding this mismatch matters for choosing prediction targets and allocating model capacity. 
It raises a concrete architectural question: what information must each predictor retain from its input for accurate denoising, and can it preserve that information when an input bottleneck reduces the representation dimension below the patch dimension?

Prior work shows that both data geometry and architecture influence prediction-target preference \citep{li2025jit,jin2026dimensionality,gagneux2026training}, but does not quantify how architectural constraints on retained input information affect denoising. 
Existing output-rank bounds establish a loss floor when predictions are confined to a restricted output subspace \citep{zheng2025rae}. 
A complementary limitation arises when the architecture discards noisy-input information: even an unrestricted decoder may then lack the information needed for accurate denoising. 
How does this error depend on the input bottleneck and noise level, even with unrestricted downstream computation, and how does an explicit input path change it? 
Once input compression is removed, how much additional width is needed to retain the noisy input while representing corrections beyond linear denoising? Appendix~\ref{app:related-literature} reviews existing results and their relation to these questions.

We address these questions through the clean-image estimate induced by a noise predictor.
For a clean image $X_1\in\mbb R^d$ and normalized Gaussian-corrupted observation $Y_t=X_1+\sigma_t\varepsilon$, the estimate is $\widehat X_1=Y_t-\sigma_t f_{\varepsilon,t}(Y_t)$. 
The explicit input term carries the corruption, and the learned branch must supply the noise contribution needed to cancel it. 
We call the requirement to retain this information a \emph{noise-transmission burden}. 
This viewpoint connects input compression, output preconditioning, and hidden width through their effects on attainable denoising accuracy.
Based on this viewpoint, our theoretical analysis and ImageNet experiments with JiT yield three linked results.

\begingroup
\setlength{\leftmargini}{1.5em}
\setlength{\labelwidth}{1em}
\setlength{\labelsep}{0.5em}
\begin{enumerate}
\item \textbf{An input bottleneck imposes a denoising-error floor for noise prediction that grows quadratically with noise.} 
Theorem~\ref{prop:input} and Corollary~\ref{cor:amplification} establish the floor $\sigma_t^2(d-r)$ when the architecture imposes a linear input bottleneck of rank $r<d$, even with unrestricted downstream computation. 
The result requires only a finite second moment for the clean data and no manifold assumption; the corresponding optimal risk for image prediction remains bounded under the same bottleneck. 
Experiments show denoising loss approaching the floor at high noise (Figure~\ref{fig:bottleneck-performance}b), with affine noise-recovery probes below the ceiling (Figure~\ref{fig:noise-ceiling}; Table~\ref{tab:exp-probe-200-predeps}).

\item \textbf{EDM alleviates the noise-transmission burden at high noise; exact preconditioning cancels it.} 
Theorem~\ref{prop:bypass} and Corollary~\ref{cor:preconditioning-clean} characterize how an explicit input path changes the denoising-error floor imposed by the input bottleneck. 
EDM suppresses it at high noise, while exact preconditioning removes it at every noise level. 
Consistent with this analysis, EDM improves FID from $412.11$ for noise prediction to $52.28$, and exact preconditioning reaches $9.15$, close to image prediction's $9.13$ (Figure~\ref{fig:bottleneck-performance}a). 
Concurrently, EDM's denoising loss approaches that of image prediction at high noise, exact preconditioning tracks image prediction across noise levels (Figures~\ref{fig:bottleneck-performance}b and~\ref{fig:bottleneck-loss-grid}), and learned skip coefficients closely track their exact-preconditioning values at moderate and high noise (Figure~\ref{fig:learned-coefficients}a).

\item \textbf{Without an input bottleneck, corrections beyond linear denoising determine how far model width must exceed the patch dimension.} 
Assuming the noisy input patch can be recovered exactly by an affine readout of the model representation, Theorem~\ref{thm:joint-width} and Corollary~\ref{cor:joint-tolerance} characterize the minimum representation width needed for a chosen level of denoising accuracy. 
Since such a representation already supports linear denoising, the additional width needed depends on the remaining nonlinear correction. 
Both the correction's magnitude and the distribution of its variance across directions matter. 
In our experiments, modest widening improves generation, denoising, and affine noisy-input recovery (Figures~\ref{fig:width-performance}a--b, \ref{fig:affine-width}, and~\ref{fig:width-probes}). 
Further widening yields larger relative denoising gains at low noise, where the required nonlinear corrections are spread across more directions than at high noise (Figures~\ref{fig:width-performance}b--c and~\ref{fig:width-performance-i128}b--c).
\end{enumerate}
\endgroup

\section{Problem setup}
\label{sec:setup}
Let $X_1\in\mbb R^d$ be a vectorized clean image drawn from a fixed distribution with $\mbb E\|X_1\|^2<\infty$. The variable $t\in[0,1]$ indexes a prescribed corruption path. We study a fixed interior time $0<t<1$, at which the signal and noise coefficients are both positive, and define the corrupted image and normalized observation by
\begin{align}
 X_t=a_tX_1+b_t\varepsilon,\qquad
 Y_t:=X_t/a_t=X_1+\sigma_t\varepsilon,\qquad
 \sigma_t:=b_t/a_t,
 \label{eq:corruption}
\end{align}
where $a_t,b_t>0$ are known deterministic coefficients and $\varepsilon\sim\N(0,I_d)$ is independent of $X_1$. 
All population risks below are evaluated at a fixed interior time; high noise means $\sigma_t\to\infty$.

Image prediction uses a network $f_{X_1,t}(y)$ to estimate $X_1$, while noise prediction uses $f_{\varepsilon,t}(y)$ to estimate $\varepsilon$. 
Under squared loss, their population-optimal predictors are the conditional means $f_{X_1,t}^{\ast}(y)=\mbb E[X_1\mid Y_t=y,t]$ and $f_{\varepsilon,t}^{\ast}(y)=\mbb E[\varepsilon\mid Y_t=y,t]$. 
At this population level, knowing either predictor determines the other by $f_{X_1,t}^{\ast}(y)=y-\sigma_t f_{\varepsilon,t}^{\ast}(y)$.
Thus both targets yield the same population-optimal clean-image estimate.
However, for a finite network, this equivalence depends on whether the network can represent each predictor equally well. 
Specifically, an image predictor directly outputs its clean-image estimate, while a noise predictor induces
\begin{align}
  \widehat X_1^{(\varepsilon)}(y):=y-\sigma_t f_{\varepsilon,t}(y).
  \label{eq:noise-induced-clean}
\end{align}
Thus $f_{\varepsilon,t}$ must supply the contribution needed to cancel unwanted Gaussian noise carried by the input term $y$.
Since Gaussian noise has nonzero variance in every ambient direction, preserving the information needed for this cancellation can require more network capacity than directly estimating the clean image, which may concentrate near low-dimensional structure \citep{li2025jit}.

The same issue arises for general targets $T_t:=\alpha_tX_1+\beta_t\varepsilon$, where $\alpha_t,\beta_t$ are known deterministic coefficients such that $\alpha_t - \beta_t/\sigma_t \neq 0$.
For example, assuming differentiable $a_t,b_t$, velocity prediction targets $\dot X_t=\dot a_tX_1+\dot b_t\varepsilon$.
Write $f_{T,t}$ for the network estimator and $f_{T,t}^{\ast}(y)=\mbb E[T_t\mid Y_t=y,t]$ for its population-optimal counterpart. 
At the population level, $f_{T,t}^{\ast}$ determines the same population-optimal clean-image estimate as $f_{X_1,t}^{\ast}$, since $f_{T,t}^{\ast}(y)=(\alpha_t-\beta_t/\sigma_t)f_{X_1,t}^{\ast}(y)+(\beta_t/\sigma_t)y$.
Nevertheless, for a finite network, the corresponding clean-image estimate is
\begin{align}
  \widehat X_1^{(T)}(y):=\{f_{T,t}(y)-(\beta_t/\sigma_t)y\}/(\alpha_t-\beta_t/\sigma_t).
  \label{eq:general-induced-clean}
\end{align}
Therefore, whenever $\beta_t\ne0$, forming an accurate clean-image estimate requires the network output $f_{T,t}(y)$ to cancel the unwanted Gaussian contribution of the explicit input term in this conversion.
If $\beta_t=0$ and $\alpha_t\ne0$, the network predicts a scaled clean image. 
Table~\ref{tab:prediction-notation} collects the relevant notation. 

\begin{table}[H]
\centering
\caption{Notation for direct target prediction at a fixed interior time.}
\label{tab:prediction-notation}
\begin{tabular}{llll}
\toprule
Prediction target & Population optimal & Network output & Induced clean image\\
\midrule
Image ($X_1$) & $f_{X_1,t}^{\ast}(y)$ & $f_{X_1,t}(y)$ & $f_{X_1,t}(y)$\\
Noise ($\varepsilon$) & $f_{\varepsilon,t}^{\ast}(y)$ & $f_{\varepsilon,t}(y)$ & $\widehat X_1^{(\varepsilon)}(y)$, \eqref{eq:noise-induced-clean}\\
General ($T_t=\alpha_tX_1+\beta_t\varepsilon$) & $f_{T,t}^{\ast}(y)$ & $f_{T,t}(y)$ & $\widehat X_1^{(T)}(y)$, \eqref{eq:general-induced-clean}\\
\bottomrule
\end{tabular}
\end{table}

We call this requirement to carry the input information needed for noise cancellation a \emph{noise-transmission burden} for a network with limited capacity. 
Specifically, a Transformer's \emph{model width} $d_h$, the dimension of each token's hidden representation, constrains its capacity to represent input and denoising information.
When this width is smaller than the patch dimension $d_p$, or when the patch passes through an intermediate embedding narrower than $d_p$, the network has an \emph{input bottleneck}.
This bottleneck discards directions of the noisy input, leaving their Gaussian components inaccessible to subsequent network computation.

We ask three questions to better illuminate this noise-transmission burden. 
First, when an input bottleneck is present, how much does it hurt noise recovery and the clean-image estimate induced by it?
Second, how much can output preconditioning, such as EDM and its flow-matching adaptation \citep{karras2022edm,kim2025simplereflow}, alleviate the effect of that bottleneck by supplying an explicit input contribution?
Finally, without such an input bottleneck, how much model width $d_h$ is needed to recover the noisy input and approximate its population denoiser?

\section{Theory}
\label{sec:claims}

\subsection{Input bottlenecks impose denoising-error floors on noise prediction}
\label{sec:input}
\begin{summarybox}
For noise prediction, an input bottleneck of rank $r<d$ induces a denoising-error floor of $\sigma_t^2(d-r)$, which grows quadratically with the noise level $\sigma_t$.
\end{summarybox}

If the input embedding discards some directions of the noisy image, how accurately can any subsequent computation recover the corruption noise?
As defined in Section~\ref{sec:setup}, an input bottleneck can arise when the model width satisfies $d_h<d_p$ or when the patch embedding passes through a narrower intermediate dimension.
In the first case, a linear projection from a $d_p$-dimensional patch to $d_h$ coordinates has rank at most $d_h$, so it discards at least $d_p-d_h$ directions.
JiT's $768\to128\to768$ bottleneck patch embedding illustrates the second case: the first linear layer reduces a patch to at most $128$ independent coordinates, and expanding back to model width $d_h=768$ does not restore the discarded information \citep[Section~4.2, Figure~4, and Table~9]{li2025jit}.

In either case, the network can be written generally as
\begin{equation}
 Y_t=X_1+\sigma_t\varepsilon
 \;\xrightarrow{\;A\;}\; AY_t
 \;\xrightarrow{\;f_{\varepsilon,t}\;}\; f_{\varepsilon,t}(AY_t),
 \label{eq:network-input-compression}
\end{equation}
where $A\in\mbb R^{r\times d}$ is a deterministic linear encoder of rank $r<d$, shared across $t$, and $\mcl F$ is the class of all measurable decoders from $\mbb R^r$ to $\mbb R^d$ with finite squared risk.
We retain the predictor notation from Section~\ref{sec:setup}, writing $f_{X_1,t}(AY_t)$ and $f_{\varepsilon,t}(AY_t)$ to make the restricted network input explicit.
All subsequent computation, including attention and the output projection, is included in the unrestricted decoder.
Note that for $M$ nonoverlapping patches of dimension $d_p$ and a shared linear patch embedding of rank $r_p$, the whole-image dimensions are $d=Md_p$ and $r=Mr_p$.
Figure~\ref{fig:dit-architecture}a illustrates this decomposition in standard DiT, together with the output preconditioning studied below.

\begin{figure}[!htbp]
\centering
\setlength{\abovecaptionskip}{3pt}
\includegraphics[width=\linewidth]{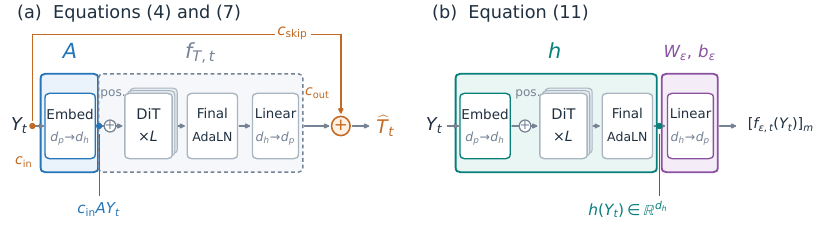}
\caption{Architectural formulations of standard DiT. (a) Linear input map $A$ and output preconditioning, \eqref{eq:network-input-compression} and \eqref{eq:general-preconditioning}. (b) Representation $h$, including final AdaLN, and the affine output layer in \eqref{eq:network-hidden-readout}.}
\label{fig:dit-architecture}
\end{figure}

An input bottleneck can be less problematic for image prediction, which only needs to estimate the clean image $X_1$ from $AY_t$.
The optimal predictor $\mbb E[X_1\mid AY_t,t]$ has denoising risk bounded above by $\mbb E\|X_1-\mbb E[X_1]\|^2$, uniformly in the noise level.
Moreover, clean images may concentrate near a structure of intrinsic dimension below the retained dimension $r$, so a suitable encoder may preserve the clean-image structure needed for denoising despite the input bottleneck.

In contrast, a noise predictor $f_{\varepsilon,t}$ must recover from $AY_t$ the noisy-input information needed to cancel the corruption carried by $Y_t$.
Precisely, its induced clean-image estimate is $\widehat X_1^{(\varepsilon)}(Y_t)=Y_t-\sigma_t f_{\varepsilon,t}(AY_t)$, following \eqref{eq:noise-induced-clean} with the network input restricted to $AY_t$. 
The explicit input term contains the intact noisy image $Y_t$, but the network $f_{\varepsilon,t}$ can use only $AY_t$ to cancel its noise.

This loss of input information limits how accurately the network can recover the noise.
A rank-$r$ encoder $A$ retains $r$ directions and discards $d-r$.
At high noise, recovering the noise in the retained $r$ directions becomes easier because $AY_t/\sigma_t=AX_1/\sigma_t+A\varepsilon$ approaches $A\varepsilon$.
At every noise level, however, the discarded $d-r$ Gaussian components are independent of $AY_t$, imposing a persistent noise-recovery loss of at least $d-r$, even for the best possible decoder.
Theorem~\ref{prop:input} quantifies it.

\begin{theorem}[Optimal noise recovery under an input bottleneck]
\label{prop:input}
Under the preceding assumptions, the optimal sampled-noise risk is at least $d-r$ at every noise level, and
\begin{equation}
 \inf_{f_{\varepsilon,t}\in\mcl F}\mbb E\|\varepsilon-f_{\varepsilon,t}(AY_t)\|^2
 =d-r+O(\sigma_t^{-2}),\qquad \sigma_t\to\infty.
 \label{eq:input-risk}
\end{equation}
\end{theorem}
This noise-recovery floor translates directly into denoising error since $\mbb E\|\widehat X_1^{(\varepsilon)}(Y_t)-X_1\|^2=\sigma_t^2\mbb E\|\varepsilon-f_{\varepsilon,t}(AY_t)\|^2$.
Corollary~\ref{cor:amplification} quantifies the induced denoising error.

\begin{corollary}[High-noise denoising error induced by an input bottleneck]
\label{cor:amplification}
Under the assumptions of Theorem~\ref{prop:input}, the induced clean-image estimate $\widehat X_1^{(\varepsilon)}(Y_t)=Y_t-\sigma_t f_{\varepsilon,t}(AY_t)$ satisfies
\begin{equation}
 \inf_{f_{\varepsilon,t}\in\mcl F}\mbb E\|\widehat X_1^{(\varepsilon)}(Y_t)-X_1\|^2
 =\sigma_t^2(d-r)+O(1),\qquad \sigma_t\to\infty.
 \label{eq:clean-amplification}
\end{equation}
\end{corollary}
The discarded directions therefore contribute at least $\sigma_t^2(d-r)$ to denoising error at every noise level, making their effect more severe at higher noise even with the best possible decoder.
The optimal denoising risk for image prediction under the same bottleneck remains bounded at every noise level.

\subsection{Output preconditioning controls the noise-transmission burden}
\label{sec:conditioning-theory}
\begin{summarybox}
Output preconditioning controls the noise-transmission burden by supplying an explicit input contribution. Exact preconditioning makes the network's target a scaled clean image, while EDM approaches this form at high noise.
\end{summarybox}

Output preconditioning combines the network output with a scaled copy of the noisy input \citep{karras2022edm,song2023consistency,li2025jit}, which can mitigate the noise-transmission burden.
To see this, consider the general target $T_t=\alpha_t X_1+\beta_t\varepsilon$ introduced in Section~\ref{sec:setup}.
With an input bottleneck $A$ as in Section~\ref{sec:input}, a preconditioned estimate takes the form
\begin{equation}
 \widehat T_t=c_{\mathrm{skip}}Y_t
 +c_{\mathrm{out}}f_{T,t}(c_{\mathrm{in}}AY_t).
 \label{eq:general-preconditioning}
\end{equation}
Here $f_{T,t}$ denotes the raw network in the preconditioned parameterization, and $\widehat T_t$ is the estimate of $T_t$ after preconditioning.
All coefficients are deterministic at the fixed time, $c_{\mathrm{in}},c_{\mathrm{out}}\ne0$, and $f_{T,t}\in\mcl F$ is unrestricted.
Since $c_{\mathrm{in}},c_{\mathrm{out}}$ can be absorbed into $f_{T,t}$, they do not change the best attainable risk.
The main effect comes from the skip coefficient $c_{\mathrm{skip}}$, which determines how much of the intact input $Y_t$ is added directly to the network output, while the network sees only $AY_t$.

Specifically, the skip term supplies $c_{\mathrm{skip}}Y_t$, leaving the scaled network $c_{\mathrm{out}}f_{T,t}(c_{\mathrm{in}}AY_t)$ to predict $T_t-c_{\mathrm{skip}}Y_t=(\alpha_t-c_{\mathrm{skip}})X_1+(\beta_t-c_{\mathrm{skip}}\sigma_t)\varepsilon$.
Unless $c_{\mathrm{skip}}=\beta_t/\sigma_t$, this residual target still contains noise in the $d-r$ directions discarded by $A$ in the input.
These components remain inaccessible to the network, yielding the following risk floor.

\begin{theorem}[Optimal target prediction under output preconditioning]
\label{prop:bypass}
Under the preceding assumptions, the optimal target risk is at least $(\beta_t-c_{\mathrm{skip}}\sigma_t)^2(d-r)$ at every noise level, and
\begin{equation}
 \inf_{f_{T,t}\in\mcl F}\mbb E\|T_t-\widehat T_t\|^2
 = (\beta_t-c_{\mathrm{skip}}\sigma_t)^2(d-r)
 +O\!\left(\left(\alpha_t-\frac{\beta_t}{\sigma_t}\right)^2
 +(\alpha_t-c_{\mathrm{skip}})^2\right),
 \quad \sigma_t\to\infty.
 \label{eq:preconditioning-floor}
\end{equation}
\end{theorem}

This bound shows how $c_{\mathrm{skip}}$ controls the noise-transmission burden: the $d-r$ discarded noise components contribute $(\beta_t-c_{\mathrm{skip}}\sigma_t)^2(d-r)$ to the risk. 
If $c_{\mathrm{skip}}=\beta_t/\sigma_t$, which we call \emph{exact preconditioning}, this term is eliminated, making the target for the scaled network output a scaled clean image, $T_t-c_{\mathrm{skip}}Y_t=(\alpha_t-\beta_t/\sigma_t)X_1$. 
Since $\alpha_t-\beta_t/\sigma_t\ne0$, the network therefore performs image prediction up to a known scaling. 

EDM alleviates the noise-transmission burden without exactly canceling it.
It uses the image target, $(\alpha_t,\beta_t)=(1,0)$, with coefficients
\[
 c_{\mathrm{skip}}=\frac{\sigma_{\mathrm{data}}^2}{\sigma_t^2+\sigma_{\mathrm{data}}^2},
 \qquad c_{\mathrm{out}}=\frac{\sigma_t\sigma_{\mathrm{data}}}{\sqrt{\sigma_t^2+\sigma_{\mathrm{data}}^2}},
 \qquad c_{\mathrm{in}}=\frac{1}{\sqrt{\sigma_t^2+\sigma_{\mathrm{data}}^2}},
\]
with fixed data scale $\sigma_{\mathrm{data}}>0$ \citep[Table~1 and Appendix~B.6, Equations~(117), (131), and (138)]{karras2022edm}. 
Theorem~\ref{prop:bypass} gives the discarded-noise term
\[
 c_{\mathrm{skip}}^2\sigma_t^2(d-r)
 =\frac{\sigma_{\mathrm{data}}^4\sigma_t^2}{(\sigma_t^2+\sigma_{\mathrm{data}}^2)^2}(d-r)
 =O(\sigma_t^{-2}).
\]
This contribution vanishes at high noise, while the remaining term in \eqref{eq:preconditioning-floor} stays bounded.
Theorem~\ref{thm:edm-target-convergence} formalizes the corresponding target limit: the raw network target $(X_1-c_{\mathrm{skip}}Y_t)/c_{\mathrm{out}}$ converges in mean square to $X_1/\sigma_{\mathrm{data}}$.
Thus the learned target approaches a scaled clean image, and the unavoidable contribution of discarded noise decays at an explicit rate. Exact preconditioning eliminates that contribution at every noise level.
This is consistent with JiT's observation that EDM-style preconditioning improves over direct noise and velocity prediction while remaining worse than image prediction in its setting.
JiT also notes the high-noise approach to image prediction as a possible explanation \citep[Appendix~B.2 and Table~10]{li2025jit}.

To illustrate the denoising consequences of these parameterizations, we convert the target prediction into a clean-image estimate. 
Rewriting the same target as $T_t=(\alpha_t-\beta_t/\sigma_t)X_1+(\beta_t/\sigma_t)Y_t$ and applying the conversion in \eqref{eq:general-induced-clean} to $\widehat T_t$ gives 
\begin{equation}
 \widehat X_1^{(T)}(Y_t)
 =\frac{\widehat T_t-(\beta_t/\sigma_t)Y_t}{\alpha_t-\beta_t/\sigma_t}
 =\underbrace{\frac{c_{\mathrm{skip}}-\beta_t/\sigma_t}{\alpha_t-\beta_t/\sigma_t}}_{\displaystyle\tilde c_{\mathrm{skip}}}Y_t
 +\underbrace{\frac{c_{\mathrm{out}}}{\alpha_t-\beta_t/\sigma_t}}_{\displaystyle\tilde c_{\mathrm{out}}}f_{T,t}(c_{\mathrm{in}}AY_t).
 \label{eq:preconditioning-conversion}
\end{equation}
The converted coefficients $\tilde c_{\mathrm{skip}}$ and $\tilde c_{\mathrm{out}}$ determine the clean-image estimate, whose risk follows from Theorem~\ref{prop:bypass}.

\begin{corollary}[Denoising risk under output preconditioning]
\label{cor:preconditioning-clean}
For $\widehat X_1^{(T)}(Y_t)=\tilde c_{\mathrm{skip}}Y_t+\tilde c_{\mathrm{out}}f_{T,t}(c_{\mathrm{in}}AY_t)$ under the preceding assumptions, every decoder satisfies $\mbb E\|\widehat X_1^{(T)}(Y_t)-X_1\|^2\ge \tilde c_{\mathrm{skip}}^2\sigma_t^2(d-r)$. If $\tilde c_{\mathrm{skip}}(t)$ stays bounded as $\sigma_t\to\infty$, then
\begin{equation}
 \inf_{f_{T,t}\in\mcl F}\mbb E\|\widehat X_1^{(T)}(Y_t)-X_1\|^2
 =\tilde c_{\mathrm{skip}}(t)^2\sigma_t^2(d-r)+O(1),
 \qquad \sigma_t\to\infty.
 \label{eq:preconditioning-clean-risk}
\end{equation}
\end{corollary}

Corollary~\ref{cor:preconditioning-clean} puts the three cases in a common form. Exact preconditioning gives $\tilde c_{\mathrm{skip}}=0$, eliminating the discarded-noise contribution, while bare noise prediction has $\tilde c_{\mathrm{skip}}=1$ and incurs the $\sigma_t^2(d-r)$ error growth. For EDM, $\tilde c_{\mathrm{skip}}=c_{\mathrm{skip}}=\sigma_{\mathrm{data}}^2/(\sigma_t^2+\sigma_{\mathrm{data}}^2)$, so this contribution vanishes at high noise and the optimal denoising risk remains bounded.

\subsection{The residual beyond linear denoising determines model width}
\label{sec:joint-width}
\begin{summarybox}
The extra model width needed for denoising is determined by the residual beyond linear denoising of the noisy patch, assuming the model representation allows affine recovery of that patch.
\end{summarybox}

When a noise predictor has access to the full noisy image, with $A=I_d$ and model width $d_h\ge d_p$, how does its width limit denoising accuracy?
To investigate this question, we focus on one output patch $m$ of dimension $d_p$, whose hidden representation has model width $d_h$ and may depend on the whole noisy image $Y_t$.
A subscript $m$ denotes the corresponding patch quantities, and $[\cdot]_m$ selects patch $m$ from a full-image prediction. In particular, $[f_{X_1,t}^{\ast}(Y_t)]_m=\mbb E[X_{1,m}\mid Y_t,t]$.
For this patch, we rewrite the network in \eqref{eq:network-input-compression} making its hidden representation and final output layer explicit:
\begin{equation}
 Y_t
 \;\xrightarrow{\;h\;}\; h(Y_t)\in\mbb R^{d_h}
 \;\xrightarrow{\;\text{final output layer}\;}\;
 [f_{\varepsilon,t}(Y_t)]_m=W_{\varepsilon}h(Y_t)+b_{\varepsilon}.
 \label{eq:network-hidden-readout}
\end{equation}
Here $h$ includes all computation before the final output layer, with $W_{\varepsilon}\in\mbb R^{d_p\times d_h}$ and $b_{\varepsilon}\in\mbb R^{d_p}$ fixed across images.
For a Transformer, $h$ includes attention and feed-forward blocks, final normalization, and modulation (Figure~\ref{fig:dit-architecture}b).
The map $W_{\varepsilon}h(Y_t)+b_{\varepsilon}$ denotes the final linear output projection, which is affine when a bias is included.

It is not clear a priori what information should be retained in the hidden representation $h(Y_t)$ for accurate denoising.
However, the identity $[f_{\varepsilon,t}^{\ast}(Y_t)]_m=(Y_{t,m}-[f_{X_1,t}^{\ast}(Y_t)]_m)/\sigma_t$ suggests a possible realization of the optimal noise predictor: $h(Y_t)$ could make both the noisy patch $Y_{t,m}$ and the denoising information for $[f_{X_1,t}^{\ast}(Y_t)]_m$ separately accessible through affine readouts.
Motivated by this, we assume that the noisy patch can be recovered affinely from $h(Y_t)$, as formalized below.

\begin{assumption}[Affine recovery of the noisy patch]
\label{ass:patch-recovery}
We consider only representations $h$ satisfying $\mbb E\|h(Y_t)\|^2<\infty$ for which there exist $W\in\mbb R^{d_p\times d_h}$ and $b\in\mbb R^{d_p}$ such that $Wh(Y_t)+b=Y_{t,m}$ with probability one.
\end{assumption}

This representation strategy is also motivated by the linear structure observed in trained denoisers.
\citet{wang2024gaussianscore} find that Gaussian score approximations describe learned scores well at moderate to high noise, while \citet[Appendix~D]{li2024hiddengaussian} report that denoisers approach the identity map at low noise.
These observations motivate studying representations that keep the noisy patch affinely accessible while encoding additional denoising corrections.

Assumption~\ref{ass:patch-recovery} already provides affine denoising features.
Under Assumption~\ref{ass:patch-recovery} and \eqref{eq:network-hidden-readout}, both the noisy patch $Y_{t,m}$ and the predicted noise $[f_{\varepsilon,t}(Y_t)]_m$ are affine in $h(Y_t)$, so their combination $\widehat X_{1,m}^{(\varepsilon)}(Y_t)=Y_{t,m}-\sigma_t[f_{\varepsilon,t}(Y_t)]_m$ is also affine in $h(Y_t)$.
Choosing $W_{\varepsilon},b_{\varepsilon}$ allows this clean estimate to realize any affine estimate of the clean patch $X_{1,m}$ based on the noisy patch $Y_{t,m}$.
In this sense, every representation satisfying Assumption~\ref{ass:patch-recovery} can support the best affine denoising estimate from the noisy patch alone, denoted by $\ell_t(Y_{t,m})$ and derived in Appendix~\ref{app:affine-baseline}.
Optimizing $h,W_{\varepsilon},b_{\varepsilon}$ within this class yields an approximation error no greater than that of $\ell_t(Y_{t,m})$.

What remains is how much additional model width $d_h-d_p$ is needed to represent the denoising features beyond this affine baseline.
We characterize these features through the residual $R_t:=[f_{X_1,t}^{\ast}(Y_t)]_m-\ell_t(Y_{t,m})$, which can reflect nonlinear dependence on the noisy patch or information from other patches.
The width question translates into how well the excess $d_h-d_p$ coordinates can approximate this residual.
Let $\lambda_j(\operatorname{Cov}(R_t))$ denote the $j$-th largest eigenvalue of the residual covariance. Theorem~\ref{thm:joint-width} characterizes the optimal approximation error through these eigenvalues.

\begin{theorem}[Optimal denoiser approximation at a fixed width]
\label{thm:joint-width}
Under Assumption~\ref{ass:patch-recovery},
\begin{align}
 \inf_{h,W_{\varepsilon},b_{\varepsilon}}\mbb E\|\widehat X_{1,m}^{(\varepsilon)}(Y_t)-[f_{X_1,t}^{\ast}(Y_t)]_m\|^2
 =\sum_{j>d_h-d_p}\lambda_j\!\left(\operatorname{Cov}(R_t)\right).
 \label{eq:joint-width-tail}
\end{align}
The infimum is attained by $h(Y_t)=(Y_{t,m},Q^\top R_t,0)$ for some $W_{\varepsilon},b_{\varepsilon}$, where $Q$ contains orthonormal eigenvectors for the largest $\min\{d_h-d_p,d_p\}$ eigenvalues of $\operatorname{Cov}(R_t)$.
\end{theorem}

At $d_h=d_p$, the best affine clean readout is $\ell_t$, with approximation error $\mbb E\|R_t\|^2$.
Additional width allows the representation to capture leading covariance directions of $R_t$, reducing the error by their variance.
Thus fewer extra coordinates are needed when the residual variance concentrates in a small number of directions.
The clean-image loss $\mbb E\|\widehat X_{1,m}^{(\varepsilon)}(Y_t)-X_{1,m}\|^2$ equals this approximation error plus the irreducible Bayes risk, as shown in Appendix~\ref{app:proof-width}.

The same spectrum determines the smallest width needed to reach a chosen approximation tolerance.

\begin{corollary}[Width required at a fixed approximation tolerance]
\label{cor:joint-tolerance}
For $\delta\ge0$, the smallest width in the preceding representation class attaining $\mbb E\|\widehat X_{1,m}^{(\varepsilon)}(Y_t)-[f_{X_1,t}^{\ast}(Y_t)]_m\|^2\le\delta$ is
\[
 d_h=d_p+\min\left\{k\in\{0,\ldots,d_p\}:\;
 \sum_{j>k}\lambda_j\!\left(\operatorname{Cov}(R_t)\right)\le\delta\right\}.
\]
\end{corollary}

For zero approximation error, the required width is $d_p+\operatorname{rank}(\operatorname{Cov}(R_t))$, so width $2d_p$ always suffices within this representation class. 

\section{Experiments}
\label{sec:experiments}
We test our theoretical predictions with experiments on ImageNet.
Sections~\ref{sec:input} and~\ref{sec:conditioning-theory} show that an input bottleneck, arising when model width $d_h$ is smaller than patch dimension $d_p$ or the input passes through a narrow patch embedding, limits noise recovery (Theorem~\ref{prop:input}) and imposes a denoising-error floor (Corollary~\ref{cor:amplification}).
EDM-style preconditioning alleviates this discarded-noise contribution at high noise, while exact preconditioning eliminates it (Theorem~\ref{prop:bypass} and Corollary~\ref{cor:preconditioning-clean}).
To test these predictions, we use JiT-B/16 \citep{li2025jit} on ImageNet256 with a $768\to128\to768$ bottleneck patch embedding, comparing prediction targets and output preconditioning through FID-10K, fixed-time denoising loss, and noise-recovery probes.

Section~\ref{sec:joint-width} shows that when $d_h\ge d_p$, the additional model width needed to approximate the population-optimal denoiser is determined by the covariance spectrum of its residual beyond the best affine denoiser, under Assumption~\ref{ass:patch-recovery} (Theorem~\ref{thm:joint-width} and Corollary~\ref{cor:joint-tolerance}).
This residual describes the correction beyond linear denoising needed at each time $t$.
To examine this width requirement, we remove the bottleneck embedding and widen noise predictors from $d_h=d_p$ to $2d_p$, using ImageNet256 with $d_p=768$ and ImageNet128 with $d_p=192$.
We also probe affine recovery of the noisy input from the hidden representation to assess empirical support for Assumption~\ref{ass:patch-recovery}.
Finally, we use a pretrained ADM denoiser \citep{dhariwal2021adm} as a proxy for the population-optimal denoiser to estimate the residual spectrum directly.

Both suites use $X_t=tX_1+(1-t)\varepsilon$, so $\sigma_t=(1-t)/t$ and small $t$ means high noise.
Denoising, probe, and residual-spectrum measurements use the ImageNet validation split, held out from backbone training.
Experimental details and additional results are provided in Appendix~\ref{app:experiments}.

\subsection{Preconditioning lowers denoising floors under input bottlenecks}
\label{sec:exp-ablation}
\label{sec:exp-shortcut}
\label{sec:exp-probe}
\textbf{EDM alleviates the noise-transmission burden; exact preconditioning cancels it.}
Noise prediction (\emph{Noise}) performs poorly under the input bottleneck, reaching FID $412.11$, compared with $9.13$ for image prediction (\emph{Image}) and $177.20$ for velocity prediction (\emph{Velocity}; Figure~\ref{fig:bottleneck-performance}a).
EDM preconditioning substantially alleviates this gap: the JiT \citep{li2025jit} and Simple ReFlow \citep{kim2025simplereflow} variants (\emph{EDM: JiT} and \emph{EDM: ReFlow}), which differ only in input scaling, both reach approximately $52.28$.
Simply supplying the noisy input through an additive shortcut (\emph{Additive}), which adds $X_t$ to the raw noise-network output, gives FID $181.09$.
Exact preconditioning (\emph{Exact}) instead reaches $9.15$ and makes the learned branch image-prediction-equivalent away from endpoint clamps.
Allowing the skip and output coefficients to be learned as functions of $t$ gives FID $9.56$, $14.99$, and $15.13$ for three variants with different initial coefficient values (\emph{Learned 1}, \emph{Learned 2}, and \emph{Learned 3}, respectively).
The learned skip approaches the exact-preconditioning condition at moderate and high noise (Figure~\ref{fig:learned-coefficients}a; Appendix~\ref{app:exp-learned-coefficients}).

\textbf{The predicted denoising floor helps explain the generation differences.}
Our theory suggests one explanation for these generation differences through a lower bound on denoising loss. 
With $r_p/d_p\le128/768=1/6$, Corollaries~\ref{cor:amplification} and~\ref{cor:preconditioning-clean} give
\[
 \frac{1}{d}\mbb E\|\widehat X_1^{(T)}(Y_t)-X_1\|^2
 \ge \tilde c_{\mathrm{skip}}^2\sigma_t^2\left(1-\frac{r_p}{d_p}\right)
 \ge \frac{5}{6}\tilde c_{\mathrm{skip}}^2\sigma_t^2.
\]
Noise prediction has $\tilde c_{\mathrm{skip}}=1$, yielding the floor $5\sigma_t^2/6$.
Exact preconditioning has $\tilde c_{\mathrm{skip}}=0$ and eliminates this discarded-noise contribution.
For EDM, $\tilde c_{\mathrm{skip}}=\sigma_{\mathrm{data}}^2/(\sigma_t^2+\sigma_{\mathrm{data}}^2)$: the contribution is most pronounced at low to intermediate noise, and vanishes as $\sigma_t\to\infty$.

These predicted floors are visible in the measured denoising losses (Figure~\ref{fig:bottleneck-performance}b).
At $t=0.05$ ($\sigma_t=19$), noise prediction has MSE $300.836$, close to the floor $300.833$; at $t=0.5$, the corresponding values are $0.83387$ and $0.83333$.
Image prediction and exact preconditioning instead give MSE $0.11016$ and $0.11019$ at $t=0.05$, and $0.01570$ and $0.01571$ at $t=0.5$.
For both EDM variants, the intermediate-noise floor is reflected in a higher loss: at $t=0.5$, their MSE is approximately $0.04747$, with a floor of $0.03333$.
At high noise, $t=0.05$, their MSE is approximately $0.1103$, close to the losses for image prediction and exact preconditioning as the EDM floor becomes negligible.
Figure~\ref{fig:bottleneck-loss-grid} compares the full loss grid for all variants.

\begin{figure}[!htbp]
\centering
\setlength{\abovecaptionskip}{3pt}
\includegraphics[width=0.9\linewidth]{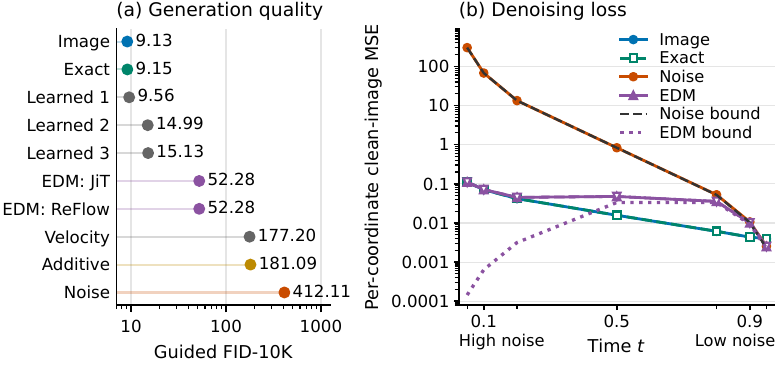}
\caption{Generation and denoising under the same input bottleneck, at epoch $200$.
(a) Guided FID-10K ($\downarrow$).
(b) Per-coordinate clean-image MSE and the discarded-noise lower bounds from Corollaries~\ref{cor:amplification} and~\ref{cor:preconditioning-clean}. 
Both EDM input scalings are plotted and nearly coincide. 
Noise prediction follows its growing floor; the EDM floor becomes negligible at high noise, where its loss approaches those for image prediction and exact preconditioning. 
FID and MSE use logarithmic scales; panel (b) uses time $t$, with $\sigma_t=(1-t)/t$. }
\label{fig:bottleneck-performance}
\end{figure}

\begingroup
\setlength{\intextsep}{4pt}
\setlength{\abovecaptionskip}{3pt}
\setlength{\belowcaptionskip}{0pt}
\Needspace{2.0in}
\begin{wrapfigure}{r}{0.38\textwidth}
\centering
\includegraphics[width=\linewidth]{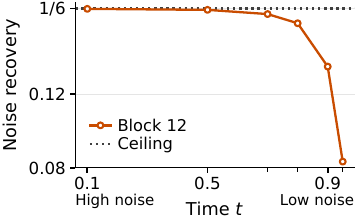}
\small
\caption{Block-$12$ noise recovery approaches but stays below the $1/6$ ceiling at high noise, consistent with Theorem~\ref{prop:input}.}
\label{fig:noise-ceiling}
\end{wrapfigure}
\textbf{Limited noise recovery makes the noise-transmission burden explicit.}
We fit affine noise probes to the noise predictor's embedding and block outputs at six times.
For a probe estimate $\widehat\varepsilon$, population explained variance is $\mathrm{EV}[\varepsilon]=1-\mbb E\|\varepsilon-\widehat\varepsilon\|^2/d$, since standard Gaussian noise has total variance $d$.
The numerator is the noise-recovery risk in Theorem~\ref{prop:input}; its floor $d-r$ gives the ceiling $\mathrm{EV}[\varepsilon]\le r/d\le1/6$.
We report the empirical counterpart using the test-noise variance.
The theorem also predicts that the optimal score approaches $r/d$ at high noise, as recovery in the retained directions becomes easier.
Figure~\ref{fig:noise-ceiling} shows the final block.
Consistently, all $78$ scores stay below $1/6$, and final-block recovery rises from $0.0833$ at low noise ($t=0.95$) to $0.1665$ at high noise ($t=0.1$).
The discarded noise still imposes the denoising floor. Table~\ref{tab:exp-probe-200-predeps} gives all layerwise scores.
\par\WFclear
\endgroup

\subsection{Widening improves denoising and affine input recovery}
\label{sec:exp-hidden-width}
We now consider noise predictors with model width $d_h\ge d_p$ and a single linear patch embedding that can preserve full input rank.
ImageNet256 with $16\times16$ RGB patches has $d_p=768$ and $256$ image tokens. 
Image-prediction controls also include widths below $d_p$. Depth and the image-loss convention remain fixed; width changes the parameter and attention-head counts.

\begin{figure}[!htbp]
\centering
\setlength{\abovecaptionskip}{3pt}
\includegraphics[width=\linewidth]{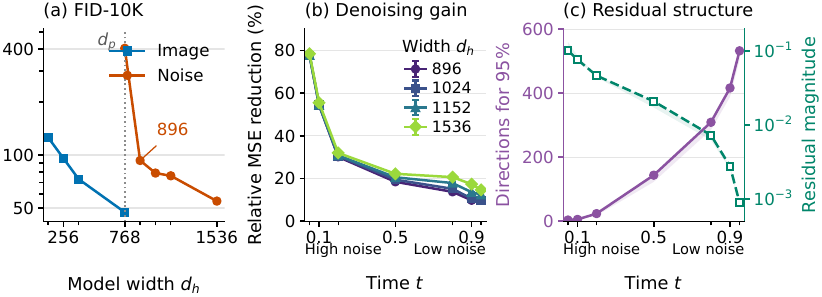}
\caption{Model width and denoiser approximation on ImageNet256. (a) Generation quality versus model width; the dotted line marks $d_p=768$. (b) Relative clean-image MSE reduction from width $768$. (c) Residual structure of a proxy for the population-optimal denoiser: purple circles give the $95\%$-variance direction count (left axis), and dashed green squares give residual magnitude (right axis, log scale). Both curves show the central patch; shading spans patch positions.}
\label{fig:width-performance}
\end{figure}

\Needspace{1.9in}
\begingroup
\setlength{\intextsep}{4pt}
\setlength{\abovecaptionskip}{3pt}
\setlength{\belowcaptionskip}{0pt}
\begin{wrapfigure}{r}{0.35\textwidth}
\centering
\includegraphics[width=\linewidth]{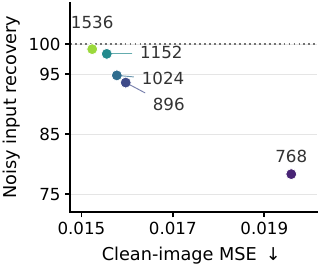}
\small
\caption{At $t=0.5$; labels: $d_h$.}
\label{fig:affine-width}
\end{wrapfigure}
\textbf{Generation and denoising improve immediately above $d_h=d_p$, with stronger affine input recovery.}
At epoch $200$, increasing the noise predictor's width from $768$ to $896$, still below $2d_p$, reduces ImageNet256 FID from $404.03$ to $93.11$ (Figure~\ref{fig:width-performance}a). Further widening gives smaller FID improvements, reaching $54.71$ at width $1536$.

These generation gains are accompanied by lower denoising loss and stronger affine noisy-input recovery from the final representation $h$ (Figure~\ref{fig:affine-width}).
At $t=0.5$, increasing width from $768$ to $896$ reduces clean-image MSE from $0.01959$ to $0.01597$ and raises noisy-input recovery from $78.34\%$ to $93.60\%$.
At width $1536$, MSE is $0.01524$ and input recovery is $99.18\%$.

This supports the input-preserving representation in Assumption~\ref{ass:patch-recovery}.
ImageNet128 shows the same association (Figures~\ref{fig:width-performance-i128}a and~\ref{fig:width-probes}b,d); Figure~\ref{fig:width-probes} shows image-prediction controls at both resolutions.
\par\WFclear
\endgroup

\textbf{Further widening gives larger relative denoising gains at low noise.}
To understand the sharp generation improvement from width $768$ to $896$ and the smaller gains from further widening, we measure relative denoising improvements at each $t$.
Recall that Theorem~\ref{thm:joint-width} characterizes how the best attainable approximation to the population-optimal denoiser improves with width under Assumption~\ref{ass:patch-recovery}.
For trained models, we estimate such approximation improvement by comparing their clean-image losses, since at a fixed $t$, the population identity gives
\[
 \underbrace{\mbb E\!\left[\|\widehat X_{1,d_p}^{(\varepsilon)}-X_1\|^2-\|\widehat X_{1,d_h}^{(\varepsilon)}-X_1\|^2\right]}_{\text{denoising-loss reduction (estimated)}}
 =\underbrace{\mbb E\!\left[\|\widehat X_{1,d_p}^{(\varepsilon)}-f_{X_1,t}^{\ast}\|^2-\|\widehat X_{1,d_h}^{(\varepsilon)}-f_{X_1,t}^{\ast}\|^2\right]}_{\text{population-denoiser approximation improvement}},
\]
where model width is indicated by an additional subscript and the common argument $Y_t$ is suppressed.
Figure~\ref{fig:width-performance}b reports the measured reduction as a percentage of the width-$768$ model's clean-image MSE at the same $t$.
At high noise ($t=0.05$), width $896$ already gives a $77.7\%$ reduction, and width $1536$ gives $78.4\%$.
At low noise ($t=0.9$), the corresponding reductions increase from $9.8\%$ to $17.4\%$.
Thus most of the high-noise benefit is obtained in the first step above $d_h=d_p$, while further widening gives a clearer relative benefit at low noise.
Measured from width $896$, widening to $1536$ reduces MSE by $3.2\%$ at $t=0.05$ and $8.4\%$ at $t=0.9$.
Figure~\ref{fig:width-performance-i128}b gives the ImageNet128 counterpart; Tables~\ref{tab:width-loss-grid-256} and~\ref{tab:width-loss-grid-128} give the full absolute-loss grids at both resolutions.

\textbf{Low-noise denoising is closer to linear, but needs corrections from many coordinates.}
Why does further widening give larger relative denoising gains at low noise? To investigate this, we inspect the covariance spectrum of the residual in Section~\ref{sec:joint-width}.
Under Assumption~\ref{ass:patch-recovery}, Theorem~\ref{thm:joint-width} and Corollary~\ref{cor:joint-tolerance} connect the concentration of this spectrum to the additional model width needed to approximate the denoiser.
A residual whose variance is spread over more directions requires more additional coordinates to capture the same fraction of that variance.
We use a pretrained ADM denoiser \citep{dhariwal2021adm} as a proxy for $f_{X_1,t}^{\ast}$, and estimate the best affine patch denoiser $\ell_t$ from clean-patch means and covariances.
Based on this, writing $\lambda_1\ge\cdots\ge\lambda_{d_p}\ge0$ for the eigenvalues of $\operatorname{Cov}(R_t)$ and taking $k\in\{0,\ldots,d_p\}$, we estimate
\[
 \text{Residual magnitude}=\frac1{d_p}\sum_{j=1}^{d_p}\lambda_j,
 \qquad
 \text{$95\%$ directions}=\min\left\{k:\sum_{j=1}^k\lambda_j\ge0.95\sum_{j=1}^{d_p}\lambda_j\right\}.
\]
Figure~\ref{fig:width-performance}c shows that the residual becomes smaller at low noise: its variance per coordinate falls by a factor of $27.7$ between $t=0.1$ and $t=0.9$.
This is consistent with the low-noise near-identity behavior reported by \citet[Appendix~D]{li2024hiddengaussian}.
At the same time, the residual becomes more dispersed: $6$ directions capture $95\%$ of its variance at high noise ($t=0.1$), compared with $416$ at low noise ($t=0.9$).
The same contrast holds on ImageNet128 (Figure~\ref{fig:width-performance-i128}c); Table~\ref{tab:adm-denoising-quality} compares the reference denoiser's accuracy with JiT.
This pattern supports a spectral interpretation of the further low-noise denoising gains from widening.

Before closing this section, we remark on the original DiT scaling experiments.
DiT-S/8 and DiT-B/8 have the same depth of $12$, with widths $384$ and $768$ and latent patch dimension $d_p=256$ \citep[Section~3.2, Table~1, and Figure~6]{peebles2023dit}.
This comparison crosses $2d_p=512$, and the improved generative performance is qualitatively consistent with additional width supporting corrections beyond linear denoising under Assumption~\ref{ass:patch-recovery}.

\section{Conclusion}
\label{sec:conclusion}
Our analysis has several limitations.
Our width characterization relies on Assumption~\ref{ass:patch-recovery}: the noisy patch must be exactly recoverable through an affine readout of the representation.
This describes one representation strategy rather than a necessary property of all successful noise predictors.
Within our unrestricted representation class, width $2d_p$ already suffices to represent the population-optimal denoiser.
Further improvements in generation from widening beyond $2d_p$ are not characterized by our current denoising-risk theory; explaining these gains is left for future work.
Moreover, although our experiments use networks shared across noise levels and evaluate both denoising and generation, the theoretical analysis treats each noise level separately, leaving open how parameter sharing changes the attainable tradeoffs.
It also characterizes denoising accuracy without establishing how errors propagate through sampling or translate into differences between generated and target distributions.

\section*{Reproducibility statement}
The repository at \url{https://github.com/hun-learning94/noiseburden} provides training and evaluation code, saved numerical results, and scripts to regenerate all tables and Figures~\ref{fig:bottleneck-performance}--\ref{fig:width-probes}. 
It includes run configurations, checkpoint identities, and instructions for supplying the external data and model weights. 
Appendix~\ref{app:proofs} contains the proofs of all theorems and corollaries, including the supporting lemmas and derivations. 
Appendix~\ref{app:experiments} documents the training and generation settings, held-out evaluation protocols, probe procedures, residual-spectrum estimation, and protocol qualifications.

\section*{AI use statement}
Generative AI tools assisted with reviewing the literature, improving the presentation of the work, developing and cleaning up code, carrying out experiments, refining proofs, providing manuscript feedback, and editing the text. 
The authors independently reviewed and verified all AI-assisted results, made the final scientific and methodological decisions, and take full responsibility for the contents of this paper.
\bibliography{references}
\bibliographystyle{iclr2027_conference}
\appendix
% Compact appendix spacing; retain the template text area and font sizes.
\raggedbottom
\setlength{\parskip}{3pt plus 1pt minus 0.5pt}
\setlength{\abovedisplayskip}{5pt plus 1pt minus 1pt}
\setlength{\belowdisplayskip}{5pt plus 1pt minus 1pt}
\setlength{\abovedisplayshortskip}{2pt plus 1pt}
\setlength{\belowdisplayshortskip}{3pt plus 1pt minus 1pt}
\setlength{\jot}{2pt}
\setlength{\topsep}{2pt plus 1pt minus 1pt}
\setlength{\partopsep}{0pt}
\setlength{\floatsep}{8pt plus 1pt minus 1pt}
\setlength{\textfloatsep}{10pt plus 2pt minus 2pt}
\setlength{\intextsep}{8pt plus 1pt minus 1pt}
\setlength{\abovecaptionskip}{4pt}
\setlength{\belowcaptionskip}{3pt}
\setcounter{topnumber}{4}
\setcounter{bottomnumber}{3}
\setcounter{totalnumber}{6}
\renewcommand{\topfraction}{0.95}
\renewcommand{\bottomfraction}{0.95}
\renewcommand{\textfraction}{0.05}
\renewcommand{\floatpagefraction}{0.8}
\makeatletter
\setlength{\@fptop}{0pt}
\setlength{\@fpsep}{8pt plus 1pt minus 1pt}
\setlength{\@fpbot}{0pt plus 1fil}
\renewcommand{\section}{\@startsection{section}{1}{\z@}
  {-1.6ex plus -0.3ex minus -0.2ex}{0.8ex plus 0.2ex}
  {\large\scshape\raggedright}}
\renewcommand{\subsection}{\@startsection{subsection}{2}{\z@}
  {-1.2ex plus -0.3ex minus -0.2ex}{0.5ex plus 0.1ex}
  {\normalsize\scshape\raggedright}}
\renewcommand{\subsubsection}{\@startsection{subsubsection}{3}{\z@}
  {-1.0ex plus -0.2ex minus -0.1ex}{0.3ex plus 0.1ex}
  {\normalsize\scshape\raggedright}}
\makeatother

\clearpage
\addtocontents{toc}{\protect\setcounter{tocdepth}{3}}
\begingroup
\renewcommand{\contentsname}{Appendix contents}
\setlength{\parskip}{0pt}
\tableofcontents
\endgroup

\section{Proofs}
\label{app:proofs}
All arguments use the model and conventions of Section~\ref{sec:setup}. We work at the fixed time of the corresponding statement and suppress this conditioning when it is unambiguous.

\subsection{Input bottlenecks}
\label{app:input-proofs}

\begin{lemma}[Projected clean-image risk under an input bottleneck]
\label{lem:compressed-clean-risk}
Under the assumptions of Section~\ref{sec:input}, for any deterministic orthogonal projector $P_A$,
\begin{equation}
 0\le\mbb E\|P_A\{X_1-\mbb E[X_1\mid AY_t,t]\}\|^2
 \le\operatorname{tr}\big(P_A\operatorname{Cov}(X_1)\big).
 \label{eq:appendix-projected-variance}
\end{equation}
The bound is uniform in the noise level.
\end{lemma}
\begin{proof}
The conditional mean $P_A\mbb E[X_1\mid AY_t,t]$ minimizes the squared error for predicting $P_AX_1$ from $AY_t$.
Its risk is therefore at most that of the constant predictor $P_A\mbb E X_1$, namely $\mbb E\|P_A(X_1-\mbb E X_1)\|^2=\operatorname{tr}(P_A\operatorname{Cov}(X_1))$, where the equality uses $P_A^\top P_A=P_A$.
\end{proof}

\subsubsection{Proof of Theorem~\ref{prop:input}}
\label{app:proof-input}
\begin{proof}
Let $P_A$ be the orthogonal projector onto the row space of $A$.
We will show that the optimal noise-recovery risk consists of a discarded-direction contribution $d-r$ and a retained-direction contribution of order $\sigma_t^{-2}$.
The conditional mean $\mbb E[\varepsilon\mid AY_t,t]$ attains the minimum squared risk, and the retained and discarded subspaces are orthogonal, so
\begin{align*}
 \inf_{f_{\varepsilon,t}\in\mcl F}\mbb E\|\varepsilon-f_{\varepsilon,t}(AY_t)\|^2
 &=\mbb E\big\|(I-P_A)\varepsilon-\mbb E[(I-P_A)\varepsilon\mid AY_t,t]\big\|^2\\
 &\quad+\mbb E\big\|P_A\varepsilon-\mbb E[P_A\varepsilon\mid AY_t,t]\big\|^2.
\end{align*}

For the first term, the projected Gaussian vectors $(I-P_A)\varepsilon$ and $P_A\varepsilon$ are independent because their cross covariance is zero.
Since $X_1$ is independent of $\varepsilon$, the discarded noise is also independent of $(X_1,P_A\varepsilon)$, and hence of $(X_1,AY_t)$ because $AY_t=AX_1+\sigma_t AP_A\varepsilon$.
It follows that
\[
 \mbb E[(I-P_A)\varepsilon\mid AY_t,t]=0,
 \qquad
 \mbb E\|(I-P_A)\varepsilon\|^2
 =\operatorname{tr}(I-P_A)=d-r.
\]
Thus the first term is exactly $d-r$ at every noise level.

For the second term, $P_A Y_t$ is determined by $AY_t$, and $P_A Y_t=P_A X_1+\sigma_t P_A\varepsilon$ gives
\[
 P_A\varepsilon-\mbb E[P_A\varepsilon\mid AY_t,t]
 =-\sigma_t^{-1}P_A\{X_1-\mbb E[X_1\mid AY_t,t]\}.
\]
The second term is therefore $\sigma_t^{-2}$ times the projected clean-image risk. By Lemma~\ref{lem:compressed-clean-risk}, it lies between zero and $\sigma_t^{-2}\operatorname{tr}(P_A\operatorname{Cov}(X_1))$.

Combining the two terms gives the exact decomposition
\[
 \inf_{f_{\varepsilon,t}\in\mcl F}\mbb E\|\varepsilon-f_{\varepsilon,t}(AY_t)\|^2
 =d-r+\sigma_t^{-2}\mbb E\big\|P_A\{X_1-\mbb E[X_1\mid AY_t,t]\}\big\|^2.
\]
The remainder is nonnegative and of order $\sigma_t^{-2}$, proving both the lower bound $d-r$ at every noise level and the high-noise expansion in \eqref{eq:input-risk}.
\end{proof}

\subsubsection{Proof of Corollary~\ref{cor:amplification}}
\label{app:proof-amplification}
\begin{proof}
For each decoder the induced clean error is exactly
\[
 \widehat X_1^{(\varepsilon)}(Y_t)-X_1
 =Y_t-\sigma_t f_{\varepsilon,t}(AY_t)-X_1
 =\sigma_t\{\varepsilon-f_{\varepsilon,t}(AY_t)\}.
\]
Multiplying the exact risk decomposition in the proof of Theorem~\ref{prop:input} by $\sigma_t^2$ therefore gives
\begin{align*}
 \inf_{f_{\varepsilon,t}}\mbb E\|\widehat X_1^{(\varepsilon)}(Y_t)-X_1\|^2
 &=\sigma_t^2(d-r)
 +\mbb E\big\|P_A\{X_1-\mbb E[X_1\mid AY_t,t]\}\big\|^2.
\end{align*}
By Lemma~\ref{lem:compressed-clean-risk}, the remainder belongs to the interval $\left[0,\operatorname{tr}(P_A\operatorname{Cov}(X_1))\right]$
uniformly in the noise level, which completes the proof.
\end{proof}

\subsection{Output preconditioning}
\label{app:preconditioning-proofs}
\subsubsection{Proof of Theorem~\ref{prop:bypass}}
\label{app:proof-preconditioning}
\begin{proof}
We will show that the optimal target risk consists of the discarded-noise contribution $(\beta_t-c_{\mathrm{skip}}\sigma_t)^2(d-r)$ and two nonnegative clean-image uncertainty terms.
Since $c_{\mathrm{in}},c_{\mathrm{out}}\ne0$, the scaled outputs $c_{\mathrm{out}}f_{T,t}(c_{\mathrm{in}}AY_t)$ range over all measurable functions of $AY_t$ with finite squared risk.
The optimum is therefore attained by the conditional mean of the residual target $T_t-c_{\mathrm{skip}}Y_t$ given $AY_t$.
Let $P_A$ project onto the row space of $A$. We split the optimal residual error into its retained and discarded components, whose squared norms add by orthogonality:
\begin{align*}
 \inf_{f_{T,t}\in\mcl F}\mbb E\|T_t-\widehat T_t\|^2
 &=\mbb E\big\|P_A\{T_t-c_{\mathrm{skip}}Y_t
       -\mbb E[T_t-c_{\mathrm{skip}}Y_t\mid AY_t,t]\}\big\|^2\\
 &\quad+\mbb E\big\|(I-P_A)\{T_t-c_{\mathrm{skip}}Y_t
       -\mbb E[T_t-c_{\mathrm{skip}}Y_t\mid AY_t,t]\}\big\|^2.
\end{align*}

For the retained component, $P_A Y_t$ is determined by $AY_t$, and substituting $P_A\varepsilon=(P_A Y_t-P_A X_1)/\sigma_t$ gives
\[
 P_A(T_t-c_{\mathrm{skip}}Y_t)
 =(\alpha_t-\beta_t/\sigma_t) P_A X_1
 +\left(\frac{\beta_t}{\sigma_t}-c_{\mathrm{skip}}\right)P_A Y_t.
\]
Its second term is observed, so subtracting the conditional mean leaves $(\alpha_t-\beta_t/\sigma_t)
 P_A\{X_1-\mbb E[X_1\mid AY_t,t]\}$, whose squared risk is a nonnegative clean-image uncertainty term.

For the discarded component, the proof of Theorem~\ref{prop:input} established that $(I-P_A)\varepsilon$ is independent of $(X_1,AY_t)$, has mean zero, and has expected squared norm $d-r$.
Consequently, subtracting the conditional mean of the discarded residual target gives
\begin{align*}
 &(I-P_A)\{T_t-c_{\mathrm{skip}}Y_t
       -\mbb E[T_t-c_{\mathrm{skip}}Y_t\mid AY_t,t]\}\\
 &\quad=(\alpha_t-c_{\mathrm{skip}})(I-P_A)
      \{X_1-\mbb E[X_1\mid AY_t,t]\}
      +(\beta_t-c_{\mathrm{skip}}\sigma_t)(I-P_A)\varepsilon.
\end{align*}
The clean uncertainty is a function of $(X_1,AY_t)$, so its expected inner product with the discarded noise is zero.
The squared risk of this component is therefore the sum of a nonnegative clean-image uncertainty term and $(\beta_t-c_{\mathrm{skip}}\sigma_t)^2(d-r)$.

Combining the retained and discarded contributions gives the exact optimal risk
\begin{align}
 \inf_{f_{T,t}\in\mcl F}\mbb E\|T_t-\widehat T_t\|^2
 &= (\beta_t-c_{\mathrm{skip}}\sigma_t)^2(d-r) +(\alpha_t-\beta_t/\sigma_t)^2
 \mbb E\big\|P_A\{X_1-\mbb E[X_1\mid AY_t,t]\}\big\|^2\nonumber\\
 &\quad+(\alpha_t-c_{\mathrm{skip}})^2
 \mbb E\big\|(I-P_A)\{X_1-\mbb E[X_1\mid AY_t,t]\}\big\|^2.
 \label{eq:preconditioning-risk}
\end{align}
The two clean-image uncertainty terms are nonnegative, proving the stated lower bound for every decoder.
By Lemma~\ref{lem:compressed-clean-risk}, their sum is bounded above by $\operatorname{tr}(\operatorname{Cov}(X_1))\big[(\alpha_t-\beta_t/\sigma_t)^2+(\alpha_t-c_{\mathrm{skip}})^2\big]$, proving the expansion in \eqref{eq:preconditioning-floor}; its implicit constant depends only on the clean-image covariance.
Note that setting $(\alpha_t,\beta_t)=(0,1)$ and $c_{\mathrm{skip}}=0$ recovers the exact noise-recovery decomposition used to prove Theorem~\ref{prop:input}.
\end{proof}

\subsubsection{Convergence of EDM's raw network target}
\label{app:edm-consequences}
\begin{theorem}[EDM's raw target approaches scaled image prediction]
\label{thm:edm-target-convergence}
Under the model of Section~\ref{sec:setup}, with EDM's coefficients from Section~\ref{sec:conditioning-theory} and fixed $\sigma_{\mathrm{data}}>0$,
\[
 \mbb E\left\|
 \frac{X_1-c_{\mathrm{skip}}Y_t}{c_{\mathrm{out}}}
 -\frac{X_1}{\sigma_{\mathrm{data}}}
 \right\|^2\longrightarrow0,
 \qquad \sigma_t\to\infty.
\]
\end{theorem}
\begin{proof}
The effective raw target admits the explicit expansion
\begin{align*}
 \frac{X_1-c_{\mathrm{skip}}Y_t}{c_{\mathrm{out}}}
 &=\frac{1-c_{\mathrm{skip}}}{c_{\mathrm{out}}}X_1
   -\frac{c_{\mathrm{skip}}\sigma_t}{c_{\mathrm{out}}}\varepsilon=\frac{\sigma_t}{\sigma_{\mathrm{data}}
              \sqrt{\sigma_t^2+\sigma_{\mathrm{data}}^2}}X_1
   -\frac{\sigma_{\mathrm{data}}}
              {\sqrt{\sigma_t^2+\sigma_{\mathrm{data}}^2}}\varepsilon.
\end{align*}
The first coefficient tends to $1/\sigma_{\mathrm{data}}$ and the second to zero. Independence and zero mean of $\varepsilon$ give the exact mean-square difference
\begin{align*}
 &\mbb E\left\|
  \frac{X_1-c_{\mathrm{skip}}Y_t}{c_{\mathrm{out}}}
  -\frac{X_1}{\sigma_{\mathrm{data}}}\right\|^2=\left(
    \frac{\sigma_t}{\sigma_{\mathrm{data}}
          \sqrt{\sigma_t^2+\sigma_{\mathrm{data}}^2}}
       -\frac{1}{\sigma_{\mathrm{data}}}\right)^2\mbb E\|X_1\|^2
       +\frac{d\sigma_{\mathrm{data}}^2}
                    {\sigma_t^2+\sigma_{\mathrm{data}}^2}
 \longrightarrow0.
\end{align*}
This proves the stated convergence to scaled image prediction.
\end{proof}

\subsubsection{Proof of Corollary~\ref{cor:preconditioning-clean}}
\label{app:proof-preconditioning-clean}
\begin{proof}
By \eqref{eq:preconditioning-conversion}, the converted clean-image estimate has the same preconditioned form, with coefficients $\tilde c_{\mathrm{skip}}$ and $\tilde c_{\mathrm{out}}$.
Apply Theorem~\ref{prop:bypass} with the image target $(\alpha_t,\beta_t)=(1,0)$ and these converted coefficients in place of $c_{\mathrm{skip}},c_{\mathrm{out}}$.
It gives the lower bound $\tilde c_{\mathrm{skip}}^2\sigma_t^2(d-r)$ for every decoder and the expansion
\[
 \inf_{f_{T,t}\in\mcl F}\mbb E\|\widehat X_1^{(T)}(Y_t)-X_1\|^2
 =\tilde c_{\mathrm{skip}}^2\sigma_t^2(d-r)
 +O\!\left(1+(1-\tilde c_{\mathrm{skip}})^2\right),
 \qquad \sigma_t\to\infty.
\]
If $\tilde c_{\mathrm{skip}}(t)$ stays bounded, the remainder is $O(1)$, proving \eqref{eq:preconditioning-clean-risk}.
\end{proof}

\subsection{Model width under affine input recovery}
\label{app:width-proofs}
\subsubsection{Derivation of the best affine denoising baseline}
\label{app:affine-baseline}
We seek the affine estimate $\ell_t(y)=Ly+u$ that minimizes $\mbb E\|X_{1,m}-LY_{t,m}-u\|^2$ at the fixed time.
Because $[f_{X_1,t}^{\ast}(Y_t)]_m=\mbb E[X_{1,m}\mid Y_t,t]$, the conditional-mean decomposition shows that replacing $X_{1,m}$ by $[f_{X_1,t}^{\ast}(Y_t)]_m$ subtracts only the irreducible Bayes risk from this objective.
Thus the best affine estimate of the clean patch is also the best affine approximation to its population-optimal denoiser.

For a fixed $L$, minimizing over $u$ makes the mean residual zero, giving $u=\mbb E[f_{X_1,t}^{\ast}(Y_t)]_m-L\mbb E Y_{t,m}$.
After substituting this intercept, the remaining problem is to fit the centered denoiser from the centered noisy patch.
Differentiating this squared-error objective with respect to $L$ gives the normal equation
\[
 L\operatorname{Cov}(Y_{t,m})
 =\operatorname{Cov}([f_{X_1,t}^{\ast}(Y_t)]_m,Y_{t,m}).
\]
Indeed, the residual $R_t=[f_{X_1,t}^{\ast}(Y_t)]_m-LY_{t,m}-u$ satisfies
\[
 \operatorname{Cov}(R_t,Y_{t,m})
 =\operatorname{Cov}([f_{X_1,t}^{\ast}(Y_t)]_m,Y_{t,m})
 -L\operatorname{Cov}(Y_{t,m})=0.
\]
Thus the optimal residual is uncorrelated with the noisy patch.
Since $\operatorname{Cov}(Y_{t,m})=\operatorname{Cov}(X_{1,m})+\sigma_t^2I_{d_p}\succ0$, it is invertible, so the unique optimal coefficients are
\begin{align}
 L&=\operatorname{Cov}([f_{X_1,t}^{\ast}(Y_t)]_m,Y_{t,m})\operatorname{Cov}(Y_{t,m})^{-1},\qquad
 u=\mbb E[f_{X_1,t}^{\ast}(Y_t)]_m-L\mbb E Y_{t,m}.
 \label{eq:affine-baseline-coefficients}
\end{align}
The tower property gives $\mbb E[f_{X_1,t}^{\ast}(Y_t)]_m=\mbb E X_{1,m}$ and $\operatorname{Cov}([f_{X_1,t}^{\ast}(Y_t)]_m,Y_{t,m})=\operatorname{Cov}(X_{1,m})$, so the coefficients depend only on the clean-patch mean and covariance and the noise level.
Consequently, the residual $R_t=[f_{X_1,t}^{\ast}(Y_t)]_m-\ell_t(Y_{t,m})$ is a square-integrable function of the whole observation $Y_t$, satisfying $\mbb E R_t=0$ and $\operatorname{Cov}(R_t,Y_{t,m})=0$.
Lastly, under Assumption~\ref{ass:patch-recovery}, the baseline $\ell_t(Y_{t,m})=LWh(Y_t)+Lb+u$ is already available as an affine readout of $h(Y_t)$.

\subsubsection{Feature decomposition under affine input recovery}
\label{app:feature-decomposition}
\begin{lemma}[Feature decomposition under affine input recovery]
\label{lem:affine-input-decomposition}
Let $h,W,b$ satisfy Assumption~\ref{ass:patch-recovery}.
Then $W$ has rank $d_p$.
After centering $h$ and $Y_{t,m}$ and using the same symbols for the centered variables, the representation can be written as
\[
 h=W^\dagger Y_{t,m}+Ng(Y_t),
 \qquad g(Y_t)\in\mbb R^{d_h-d_p},
\]
where $W^\dagger$ is a right inverse of $W$, the columns of $N$ form an orthonormal basis of $\ker W$, and $g$ is a square-integrable function of $Y_t$.
\end{lemma}
\begin{proof}
The affine recovery identity $Wh(Y_t)+b=Y_{t,m}$ gives $\operatorname{Cov}(Y_{t,m})=W\operatorname{Cov}(h(Y_t))W^\top$.
Since $\operatorname{Cov}(Y_{t,m})=\operatorname{Cov}(X_{1,m})+\sigma_t^2I_{d_p}\succ0$, $W$ has rank $d_p$ and $\ker W$ has dimension $d_h-d_p$.
Thus a right inverse $W^\dagger$ exists.

For the centered variables, $Wh=Y_{t,m}$, so $W(h-W^\dagger Y_{t,m})=0$.
The difference $h-W^\dagger Y_{t,m}$ therefore lies in $\ker W$ and can be written as $Ng(Y_t)$, with $g(Y_t)=N^\top(h-W^\dagger Y_{t,m})$.
This gives the claimed decomposition.
The function $g$ is square-integrable because it is a fixed linear combination of the square-integrable variables $h$ and $Y_{t,m}$.
Thus affine recovery separates the representation into a linear function of the noisy patch and $d_h-d_p$ additional feature coordinates.
\end{proof}

\subsubsection{Proof of Theorem~\ref{thm:joint-width}}
\label{app:proof-width}
\begin{proof}
We will first show that every admissible representation has approximation error at least the spectral tail in \eqref{eq:joint-width-tail}, and then construct one that attains this bound.
The key is to separate the features needed for affine recovery of the noisy patch from the remaining $d_h-d_p$ feature coordinates, and determine how many output directions those remaining coordinates can supply.

First, we express the problem in terms of a clean-image readout.
Under Assumption~\ref{ass:patch-recovery}, the noise-output head in \eqref{eq:network-hidden-readout} induces $\widehat X_{1,m}^{(\varepsilon)}(Y_t)=Ch(Y_t)+c$, where $C=W-\sigma_tW_{\varepsilon}$ and $c=b-\sigma_tb_{\varepsilon}$.
Conversely, any $C,c$ determine a noise-output head through $W_{\varepsilon}=(W-C)/\sigma_t$ and $b_{\varepsilon}=(b-c)/\sigma_t$.
Thus optimizing the noise-output layer is equivalent to optimizing an arbitrary affine clean readout of the same representation.
We therefore use $C,c$ to study its approximation error directly.

For the lower bound, fix any admissible representation and center $h$, $Y_{t,m}$, and the population denoiser, absorbing their means into the affine intercepts and using the same symbols for the centered variables.
Since any constant mean deviation contributes an additional nonnegative squared norm, a lower bound for the centered error also bounds the original error.
Lemma~\ref{lem:affine-input-decomposition} decomposes such $h$ into the recovered input and $d_h-d_p$ additional feature coordinates.
Applying the clean readout to this decomposition gives
\[
 Ch=CW^\dagger Y_{t,m}+CN g(Y_t).
\]
The first term is linear in the noisy patch.
The second can use arbitrary nonlinear features of the whole image, but its output always lies in $S=\operatorname{col}(CN)$, whose dimension is at most $d_h-d_p$.
In patch-output space $\mbb R^{d_p}$, $S$ contains the directions that the additional features can affect through the readout $CN$.

Since the additional features contribute only within $S$, we obtain a lower bound by examining the prediction error in $S^\perp$, where their contribution vanishes regardless of how expressive $g$ is.
The orthogonal complement $S^\perp$ consists of the patch-output directions perpendicular to every direction in $S$, whose predicted components the additional features cannot change.
Write the centered population denoiser as $[f_{X_1,t}^{\ast}(Y_t)]_m=LY_{t,m}+R_t$, using the affine baseline from Appendix~\ref{app:affine-baseline}.
The approximation error is then
\[
 Ch-[f_{X_1,t}^{\ast}(Y_t)]_m
 =(CW^\dagger-L)Y_{t,m}+CN g(Y_t)-R_t.
\]
Projecting onto $S^\perp$ removes $CN g(Y_t)$.
The remaining affine patch term cannot cancel the projected residual in mean square, because $\operatorname{Cov}(R_t,Y_{t,m})=0$.
Writing $P_{S^\perp}$ for the orthogonal projector onto $S^\perp$, we obtain
\begin{align*}
 \mbb E\|Ch-[f_{X_1,t}^{\ast}(Y_t)]_m\|^2
 &\ge\mbb E\|P_{S^\perp}\{(CW^\dagger-L)Y_{t,m}-R_t\}\|^2\\
 &=\mbb E\|P_{S^\perp}(CW^\dagger-L)Y_{t,m}\|^2
   +\mbb E\|P_{S^\perp}R_t\|^2\\
 &\ge\operatorname{tr}\big(P_{S^\perp}\operatorname{Cov}(R_t)\big).
\end{align*}
The equality uses the vanishing cross term; the last inequality drops the nonnegative error from the affine patch contribution.
Every admissible readout therefore incurs at least the residual variance outside its additional-feature output subspace.

Let $k=d_h-d_p\ge0$.
We now bound the variance captured by a subspace $S$ of the patch-output space $\mbb R^{d_p}$ with dimension at most $k$.
Let $q_1,\ldots,q_{d_p}$ be orthonormal eigenvectors of $\operatorname{Cov}(R_t)$, ordered by decreasing eigenvalue.
The total residual variance is $\operatorname{tr}(\operatorname{Cov}(R_t))=\sum_{j=1}^{d_p}\lambda_j(\operatorname{Cov}(R_t))$.
Writing $P_S$ for the orthogonal projector onto $S$, the variance captured by $S$ is
\[
 \operatorname{tr}\big(P_S\operatorname{Cov}(R_t)\big)
 =\sum_{j=1}^{d_p}\lambda_j\big(\operatorname{Cov}(R_t)\big)\|P_Sq_j\|^2.
\]
Each weight $\|P_Sq_j\|^2$ lies in $[0,1]$, and the weights sum to $\operatorname{tr}(P_S)=\dim S\le k$.
Since the eigenvalues are nonnegative and ordered from largest to smallest, this weighted sum is at most the sum of the largest $\min\{k,d_p\}$ eigenvalues.
Subtracting the captured variance from the total therefore gives
\begin{align*}
 \operatorname{tr}\big(P_{S^\perp}\operatorname{Cov}(R_t)\big)
 &=\operatorname{tr}\big(\operatorname{Cov}(R_t)\big)
   -\operatorname{tr}\big(P_S\operatorname{Cov}(R_t)\big) \ge\sum_{j>k}\lambda_j\big(\operatorname{Cov}(R_t)\big).
\end{align*}
For $k\ge d_p$, the sum is empty and equals zero.
Together with the preceding error bound, this proves that no admissible representation and output layer can improve on the right side of \eqref{eq:joint-width-tail}.

It remains to show that this lower bound is attainable.
Return to the original, uncentered variables, and let $Q$ contain orthonormal eigenvectors for the largest $\min\{k,d_p\}$ eigenvalues of $\operatorname{Cov}(R_t)$.
Store the noisy patch and the corresponding residual coordinates in
\[
 h(Y_t)=\big(Y_{t,m},\ Q^\top R_t,\ 0\big)\in\mbb R^{d_h},
\]
where the final zero block pads unused coordinates.
The first $d_p$ coordinates recover the noisy patch exactly and provide the affine baseline $\ell_t(Y_{t,m})$.
The remaining coordinates provide the projection $QQ^\top R_t$, so an affine clean readout can produce
\[
 \widehat X_{1,m}^{(\varepsilon)}(Y_t)=\ell_t(Y_{t,m})+QQ^\top R_t.
\]
Its error is exactly the omitted residual $-(I-QQ^\top)R_t$, giving
\[
 \mbb E\|\widehat X_{1,m}^{(\varepsilon)}(Y_t)-[f_{X_1,t}^{\ast}(Y_t)]_m\|^2
 =\sum_{j>k}\lambda_j\big(\operatorname{Cov}(R_t)\big).
\]
All feature blocks are square-integrable functions of $Y_t$, and the readout matrices are fixed across images.
The correspondence at the start of the proof supplies noise-output coefficients $W_{\varepsilon},b_{\varepsilon}$ inducing this clean estimate.
The construction therefore belongs to the stated class and attains the lower bound, proving the theorem.

Finally, the standard conditional-mean decomposition translates denoiser approximation into clean-image loss:
\begin{align*}
 \mbb E\|\widehat X_{1,m}^{(\varepsilon)}(Y_t)-X_{1,m}\|^2
 &=\mbb E\|[f_{X_1,t}^{\ast}(Y_t)]_m-X_{1,m}\|^2 +\mbb E\|\widehat X_{1,m}^{(\varepsilon)}(Y_t)-[f_{X_1,t}^{\ast}(Y_t)]_m\|^2.
\end{align*}
The first term is the irreducible Bayes risk and is independent of the representation.
Thus the spectral tail also gives the smallest excess clean-image loss above that risk within the stated class.
\end{proof}

\subsubsection{Proof of Corollary~\ref{cor:joint-tolerance}}
\label{app:proof-width-tolerance}
\begin{proof}
For width $d_h=d_p+k$, Theorem~\ref{thm:joint-width} gives the attained minimum approximation error $\sum_{j>k}\lambda_j(\operatorname{Cov}(R_t))$.
This tail sum is nonincreasing in $k$ and is zero at $k=d_p$.
The smallest width attaining tolerance $\delta$ is therefore $d_p$ plus the smallest $k\in\{0,\ldots,d_p\}$ for which the tail sum is at most $\delta$.
For $\delta=0$, all positive residual eigenvalues must be retained, giving $d_h=d_p+\operatorname{rank}(\operatorname{Cov}(R_t))$.
\end{proof}

\section{Experimental protocols and additional results}
\label{app:experiments}
The two subsections follow Sections~\ref{sec:exp-ablation} and~\ref{sec:exp-hidden-width}, respectively.
Fixed-time denoising, probe, and residual-spectrum evaluations use the official ImageNet validation split, held out from backbone training. Fitting, tuning, and test images are disjoint within each diagnostic.

\subsection{Preconditioning lowers denoising floors under input bottlenecks}
\label{app:exp-targets}
\subsubsection{Training, preconditioning, and generation}
\label{app:exp-training}
\label{app:exp-shortcuts}
\label{app:exp-edm}
\textbf{Training.}
The bottleneck suite uses $12$ Transformer blocks and model width $768$, with the patch embedding specified in Section~\ref{sec:exp-ablation}.
Training images are center-cropped, randomly flipped, and normalized to $[-1,1]$.
We use seed $0$, unit Gaussian noise, $\operatorname{logit}(t)\sim\mathcal N(-0.8,0.8^2)$, label-drop probability $0.1$, and zero attention and projection dropout.
The effective batch size is $1024$. AdamW uses $(\beta_1,\beta_2)=(0.9,0.95)$ and zero weight decay; the learning rate warms up for five epochs to $2\times10^{-4}$ and remains constant thereafter.
Each configuration has one training run.

\textbf{Implemented image-loss convention.}
Both experimental suites express the image-space objective as a weighted velocity loss, with conversion denominators clamped at $\tau=0.05$.
Let $c_-(t)=\max(1-t,\tau)$ and $c_+(t)=\max(t,\tau)$. For complete image or noise estimates, the implementation uses
\begin{align}
 v_{\rm target}&=(X_1-X_t)/c_-(t),\nonumber\\
 \widehat v_{\rm image}&=(\widehat X_1-X_t)/c_-(t),\nonumber\\
 \widehat v_\varepsilon&=(X_t-\widehat\varepsilon)/c_+(t),\nonumber\\
 \mathcal L&=\mbb E\big[(1-t)^2\|\widehat v-v_{\rm target}\|^2\big].
 \label{eq:width-implemented-loss}
\end{align}
Pixel dimensions are averaged in the code. Velocity prediction targets $X_1-\varepsilon$ and supplies $\widehat v$ directly; its induced clean estimate is $X_t+(1-t)\widehat v$.
For $\tau\le t\le1-\tau$, this objective equals induced clean-image squared error. Endpoint clamps modify that identity.
Exact preconditioning and its learned initialization also replace $1-t$ by $\max(1-t,\tau)$ in their output coefficients.
All seven fixed-time evaluation points lie in the interval where the implemented and analytic conversions agree, up to numerical precision.
Time and class conditioning are implicit below; the observed class label is included in the population conditioning.

\textbf{Fixed and learned noise-output coefficients.}
In implementation coordinates, the shortcut models output
\begin{equation}
 \widehat\varepsilon=\frac{c_{\mathrm{skip}}}{t}X_t+c_{\mathrm{out}}f_{\varepsilon,t}(AX_t).
 \label{eq:exp-shortcut}
\end{equation}
Since $X_t=tY_t$, this is the parameterization in Section~\ref{sec:conditioning-theory} with $c_{\mathrm{in}}=t$.
Table~\ref{tab:exp-coeff} lists the fixed choices and their converted clean-image skip.
The additive choice leaves $\tilde c_{\mathrm{skip}}=t$, so its discarded-noise contribution is $(1-t)^2(d-r)$ by Corollary~\ref{cor:preconditioning-clean}.

\begin{table}[!htbp]
\centering
\caption{Fixed noise-output coefficients and the resulting clean-image skip. The learned branch receives $AX_t$, so $c_{\mathrm{in}}=t$.}
\label{tab:exp-coeff}
\small
\begin{tabular}{lccc}
\toprule
Variant & $c_{\mathrm{skip}}$ & $c_{\mathrm{out}}$ & $\tilde c_{\mathrm{skip}}$\\
\midrule
Noise & $0$ & $1$ & $1$\\
Additive & $t$ & $1$ & $t$\\
Exact & $1/\sigma_t$ & $-1/\sigma_t$ & $0$\\
\bottomrule
\end{tabular}
\end{table}

A coefficient MLP uses eight sine/cosine frequency pairs, a hidden layer of $64$ units, and a SiLU activation.
Its zero-initialized final layer adds learned corrections to $c_{\mathrm{skip}}/t$ and $c_{\mathrm{out}}$ in \eqref{eq:exp-shortcut}.
Learned 1, Learned 2, and Learned 3 initialize $(c_{\mathrm{skip}},c_{\mathrm{out}})$ to $(1/\sigma_t,-1/\sigma_t)$, $(0,1)$, and $(t,1)$, respectively: exact, noise, and additive initialization.
The same endpoint clamps apply, and the trained skip is evaluated at each time when computing a denoising bound.

\textbf{The two EDM input scalings.}
Both variants use the image target and the EDM output coefficients in Section~\ref{sec:conditioning-theory}, with $\sigma_{\mathrm{data}}=0.5$ \citep[Table~1 and Appendix~B.6]{karras2022edm}.
\emph{EDM: JiT} retains $AX_t$ as the raw input, so $c_{\mathrm{in}}=t$ \citep[Appendix~B.2, Equations~(7)--(8)]{li2025jit}.
\emph{EDM: ReFlow} uses $AX_t/\sqrt{\sigma_{\mathrm{data}}^2t^2+(1-t)^2}$, giving $c_{\mathrm{in}}=1/\sqrt{\sigma_{\mathrm{data}}^2+\sigma_t^2}$ \citep[Appendix~D.2, Equations~(22)--(25) and (35)]{kim2025simplereflow}.
In raw $X_t$ coordinates, the explicit image-output coefficient is $c_{\mathrm{skip}}/t$.
The backbone, output coefficients, optimizer, time sampling, image-loss convention, and evaluation protocol are otherwise shared.
The $c_{\mathrm{out}}^{-2}$ weighting in \citet[Appendix~D.2.2, Equation~(23)]{kim2025simplereflow} is not used in this input-scaling comparison.

\textbf{Generation.}
Generation uses Heun's method with $50$ steps, an Euler final step, and $10{,}000$ generated images.
Guidance scale $2.9$ is active over $[0.1,1.0]$; FID uses the JiT ImageNet256 reference statistics.
Evaluations use epoch-$200$ EMA1 weights with per-update decay $0.9999$.
Nearby FID values are descriptive single-run measurements, without a run-to-run uncertainty estimate.

\subsubsection{Learned versus exact preconditioning coefficients}
\label{app:exp-learned-coefficients}
We extract the coefficient MLPs from the epoch-$200$ EMA1 checkpoints used for generation and denoising, checking their implementation and checkpoint hashes.
Figure~\ref{fig:learned-coefficients} evaluates a $901$-point uniform grid in $[0.05,0.95]$, supplemented by the seven denoising-evaluation times and checked against their saved coefficients. This interval avoids changes from endpoint clamps.

For a noise target, $\tilde c_{\mathrm{skip}}=1-\sigma_t c_{\mathrm{skip}}$, so a unit skip ratio $c_{\mathrm{skip}}/(1/\sigma_t)$ marks the exact-preconditioning condition.
All three variants have $|\tilde c_{\mathrm{skip}}|<0.0045$ over $0.05\le t\le0.5$.
At $t=0.9$, the corresponding values are $0.00566$, $0.899$, and $0.834$ for Learned 1, Learned 2, and Learned 3, respectively, showing their different low-noise behavior.
The output ratio $c_{\mathrm{out}}/(-1/\sigma_t)$ need not equal one: rescaling or changing the sign of the learned branch can be absorbed into $c_{\mathrm{out}}$. Learned 2 and Learned 3 have positive $c_{\mathrm{out}}$, while the exact coefficient is negative.

\begin{figure}[!htbp]
\centering
\includegraphics[width=0.86\linewidth]{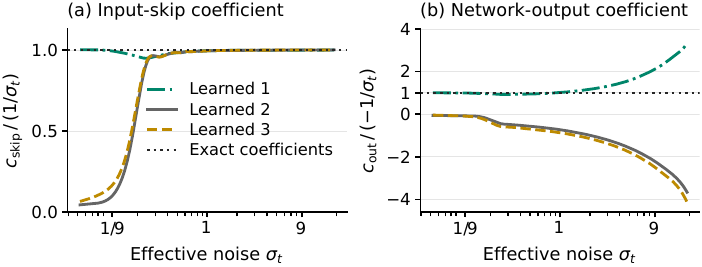}
\caption{Learned coefficients relative to exact preconditioning at epoch $200$ (EMA1). The dotted line marks exact coefficients. (a) The skip ratio approaches $1$ at moderate and high noise. (b) Output ratios can differ in sign and scale because the learned branch can absorb these changes.}
\label{fig:learned-coefficients}
\end{figure}

\subsubsection{Fixed-time denoising losses}
\label{app:exp-native-grid}
\textbf{Checkpoints and conversions.}
We evaluate epoch-$200$ EMA1 weights for all ten bottleneck variants and all sixteen width checkpoints at $t\in\{0.05,0.1,0.2,0.5,0.8,0.9,0.95\}$.
Predictions include the complete skip and output scale before the clean-image conversion defined in Section~\ref{sec:setup}.
Evaluations use true class labels, no classifier-free guidance, and no clipping. We average clean-image and sampled-noise squared error per coordinate and per image, and check their conversion identity numerically.

\textbf{Images and numerical evaluation.}
We select ten images per class without replacement from the official ImageNet validation split, using seed $20261025$.
The resulting $10000$ images are ordered in ten blocks of $1000$, each containing one image per class in randomized class order.
The final two blocks supply $2000$ evaluation images, two per class; validation labels use the same class order as backbone training.
Images receive native-resolution center crops, no random flips, and $[-1,1]$ normalization.
Each image's standard Gaussian draw uses seed $23102026$ plus its zero-based position in this $10000$-image fixture. Draws are reused across times and models at each resolution.
The resolutions share image identities, with resolution-specific Gaussian arrays. The probe fixtures below use separate corruption draws.
The frozen backbone uses BF16 and target conversion uses FP32.
We retain per-image errors and compute $95\%$ percentile intervals with $1000$ image-bootstrap resamples, conditional on the trained checkpoints.
The bottleneck floor uses the rank upper bound $r_p\le128$; a lower actual rank would strengthen it. Curves connect only the seven measured times.

At $t=0.05$, the $95\%$ MSE intervals for noise prediction, image prediction, and exact preconditioning are $[300.794,300.879]$, $[0.10810,0.11223]$, and $[0.10813,0.11229]$, respectively. The noise-prediction interval contains the predicted floor $300.833$.
Figure~\ref{fig:bottleneck-loss-grid} shows all bottleneck results, including variants omitted from Figure~\ref{fig:bottleneck-performance}b.

\begin{figure}[!htbp]
\centering
\includegraphics[width=\linewidth]{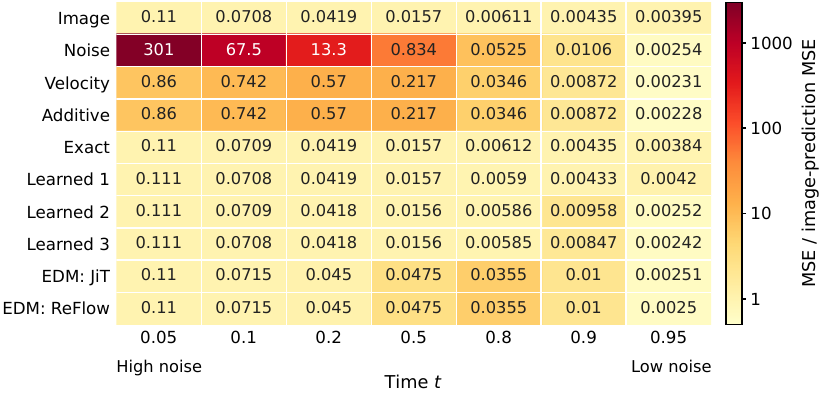}
\caption{Full denoising-loss comparison under the input bottleneck at epoch $200$ (EMA1), using $2000$ common validation images without guidance. Cell values give per-coordinate clean-image MSE to three significant figures; color gives the ratio to the image-prediction MSE at the same $t$, on a logarithmic scale. Small $t$ means high noise.}
\label{fig:bottleneck-loss-grid}
\end{figure}

\subsubsection{Layerwise noise-recovery probes}
\label{app:exp-probes}
\label{app:layerwise-probes}
\textbf{Layerwise affine probes.}
We use the epoch-$200$ noise-prediction checkpoint with EMA1 decay $0.9999$ at $t=0.1,0.5,0.7,0.8,0.9,0.95$.
Layer $0$ is the first block's input, including embedding and positional terms; layers $1$--$12$ are block outputs, with class-context tokens removed.
Block $12$ precedes final normalization, modulation, and output projection; it differs from the final-head input used in the width probes.
The first $8000$ validation images from Appendix~\ref{app:exp-native-grid} fit the probes, and the final $2000$ form the disjoint test set, giving eight fitting and two test images per class. Images receive $256\times256$ center crops, no flips, and $[-1,1]$ normalization.
Independent fitting and evaluation Gaussian streams use seeds $117$ and $991$. All layers receive common images, labels, and corruptions; corruption draws are also shared across times.

Each layer/time pair has a separate affine decoder of the noise patch, with its matrix and intercept shared across patch positions.
The ridge coefficient is $10^{-3}$ times the mean feature-Gram diagonal; the intercept receives a $10^{-8}$ diagonal term for numerical stability.
Backbone inference uses FP32 with TF32 disabled; sufficient statistics and the linear solve use FP64. All six times were evaluated together on one H200.
Scores use the explained-variance definition in Section~\ref{sec:exp-ablation}, with noise variance estimated on the test set by centering each within-patch coordinate across images and patch positions.
They include prediction bias; finite-test normalization can fluctuate around the population bound.
Table~\ref{tab:exp-probe-200-predeps} retains all layers, complementing the block-$12$ curve in Figure~\ref{fig:noise-ceiling}.

\begin{table}[H]
\centering
\caption{Layerwise affine noise-recovery scores for the noise predictor at epoch $200$ (EMA1; $\uparrow$), evaluated together on one H200. Scores use $2000$ validation test images and finite-test noise-variance normalization. Layer $0$ is the embedding; all values are below $1/6$.}
\label{tab:exp-probe-200-predeps}
\small
\begin{tabular}{rrrrrrr}
\toprule
Layer & $t=0.1$ & $t=0.5$ & $t=0.7$ & $t=0.8$ & $t=0.9$ & $t=0.95$\\
\midrule
0 & 0.16662 & 0.16611 & 0.16384 & 0.15873 & 0.13403 & 0.08302\\
1 & 0.16662 & 0.16611 & 0.16383 & 0.15872 & 0.13381 & 0.08163\\
2 & 0.16662 & 0.16610 & 0.16383 & 0.15871 & 0.13387 & 0.08206\\
3 & 0.16662 & 0.16610 & 0.16383 & 0.15871 & 0.13389 & 0.08226\\
4 & 0.16662 & 0.16610 & 0.16383 & 0.15872 & 0.13392 & 0.08225\\
5 & 0.16662 & 0.16610 & 0.16383 & 0.15871 & 0.13390 & 0.08224\\
6 & 0.16662 & 0.16610 & 0.16382 & 0.15871 & 0.13393 & 0.08247\\
7 & 0.16661 & 0.16610 & 0.16382 & 0.15871 & 0.13393 & 0.08247\\
8 & 0.16661 & 0.16609 & 0.16382 & 0.15871 & 0.13394 & 0.08250\\
9 & 0.16661 & 0.16609 & 0.16382 & 0.15871 & 0.13394 & 0.08252\\
10 & 0.16660 & 0.16608 & 0.16382 & 0.15871 & 0.13396 & 0.08259\\
11 & 0.16656 & 0.16604 & 0.16378 & 0.15868 & 0.13397 & 0.08258\\
12 & 0.16651 & 0.16597 & 0.16368 & 0.15868 & 0.13509 & 0.08334\\
\bottomrule
\end{tabular}
\end{table}

\subsection{Widening improves denoising and affine input recovery}
\label{app:bingshuo-protocol}
\subsubsection{Model configurations and generation results}
\textbf{Architecture and training.}
The sweep uses $16\times16$ RGB patches at resolution $256$ and $8\times8$ patches at resolution $128$, with $12$ blocks and $256$ image tokens in every model.
A single learned patch projection maps $\mbb R^{d_p}$ to $\mbb R^{d_h}$.
Noise widths are $768,896,1024,1152,1536$ at resolution $256$ and $192,256,320,384$ at resolution $128$; image widths are $128,256,384,768$ and $64,128,192$, respectively.
Noise predictors have no forced reduction below patch dimension at the input, while image predictors with $d_h<d_p$ retain an input bottleneck.
Attention uses the full model width in its query, key, value, and output projections; no output bypass is added.
Head dimension is $64$, except for the width-$128$, $256$, and $384$ ImageNet256 image predictors, which use head dimension $32$.
The head count is $d_h$ divided by head dimension, so widening changes parameter and head counts as well as representation dimension.

Training follows the preprocessing, optimizer, batch size, schedule, time sampling, and seed-$0$ conventions in Appendix~\ref{app:exp-training}, with loss \eqref{eq:width-implemented-loss}.
Each configuration has one run; checkpoint identity is checked by file hashes.
The ImageNet128 image-width-$64$ checkpoint comes from a fresh matched seed-$0$ run and is shared by generation, native-loss, and input-recovery evaluations.

\textbf{Generation.}
Generation uses Heun's method with $50$ steps, $10{,}000$ images, guidance scale $1$, and sampling seed $0$.
FID uses resolution-specific ImageNet reference statistics; generated samples are not paired across models.
All reported FIDs use epoch-$200$ EMA1 weights with decay $0.9999$ and one GPU with generation batch size $8$; ImageNet256 evaluation used an RTX 5000 Ada GPU.
Figures~\ref{fig:width-performance}a and~\ref{fig:width-performance-i128}a show the complete generation comparison at both resolutions.
These are single-run FID-10K results, without run-to-run uncertainty estimates.

\begin{figure}[H]
\centering
\includegraphics[width=\linewidth]{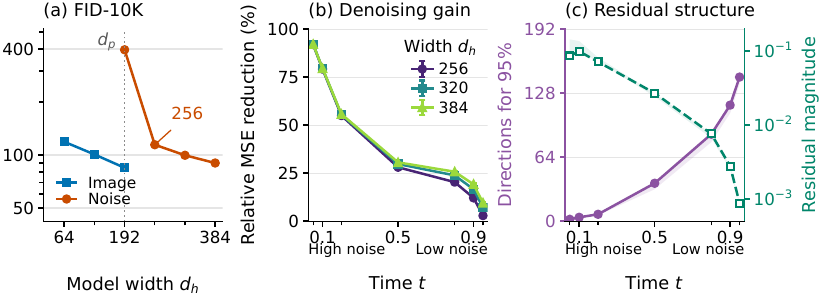}
\caption{ImageNet128 counterpart to Figure~\ref{fig:width-performance}, with patch dimension $d_p=192$. (a) FID-10K versus model width. (b) Relative clean-image MSE reduction from width $192$. (c) Residual structure of a proxy for the population-optimal denoiser. Colors, axes, and shading follow Figure~\ref{fig:width-performance}.}
\label{fig:width-performance-i128}
\end{figure}

\subsubsection{Affine input-recovery probes and controls}
\label{app:exp-affine-recovery}
\label{app:bingshuo-readouts}
\textbf{Features and images.}
We use all sixteen epoch-$200$ EMA1 checkpoints at $t=0.5$.
The feature $h$ is the input to the final linear output projection, after final normalization and modulation, with class-context tokens excluded.
We use the validation fixture from Appendix~\ref{app:exp-native-grid}, with deterministic resolution-specific center crops normalized to $[-1,1]$ and no augmentation.
A fixed permutation of the first $8000$ images, with seed $251025$, supplies $384$ fitting and $128$ tuning images; the final $2000$ fixture images form the disjoint test set.
Within each resolution, all models receive identical images, labels, and Gaussian corruptions.
A fixed CUDA Gaussian stream with seed $99003$ supplies the fitting, tuning, and first $128$ test fields. For these fields at resolution $128$, each disjoint $2\times2$ group of the resolution-$256$ field is summed and divided by $2$.
The remaining $1872$ test fields use independent standard Gaussian pixels from CPU generators with seeds $99003+j$, where $j=640,\ldots,2511$ is the example's index in the combined fitting, tuning, and test sequence, separately at each resolution.
Input recovery and native denoising use the same $2000$ test images and corruptions. These image identities match the fixed-time denoising grid, but the corruption draws differ.

\textbf{Affine fitting and score.}
The probe matrix is shared across all $256$ patch positions, with a position-dependent intercept estimated from fitting images only.
Tuning MSE selects a ridge penalty from $\{0,10^{-8},10^{-6},10^{-4},10^{-2},1\}$ times the mean eigenvalue of the fitting feature covariance; zero penalty uses a pseudoinverse with relative eigenvalue cutoff $10^{-12}$.
The backbone uses BF16; fitting and score accumulation use FP64.
For the fitted readout $\widehat Y_{t,m}$, the population input-recovery score is $100\big(1-\mbb E\|Y_{t,m}-\widehat Y_{t,m}\|^2/\mbb E\|Y_{t,m}-\mbb E[Y_{t,m}]\|^2\big)$.
Its empirical counterpart pools squared errors and target variances across images, positions, and coordinates, centering each within-patch coordinate across images and patch positions. It includes prediction bias and equals $100R^2$.
The implementation targets $X_{t,m}=tY_{t,m}$, which gives the same affine classes and normalized scores at fixed $t>0$.
Only the input recovery is probed: denoising MSE uses the trained output head, true class labels, and no guidance.

Figure~\ref{fig:width-probes} adds the ImageNet128 counterpart and image-prediction controls to the ImageNet256 noise-predictor comparison in Figure~\ref{fig:affine-width}.
The controls show that affine input preservation is not a universal requirement for accurate denoising across prediction targets.
The probes measure finite-sample affine accessibility at one time, without establishing exact recovery, its necessity for noise prediction, or a causal effect on denoising.

\begin{figure}[H]
\centering
\includegraphics[width=\linewidth]{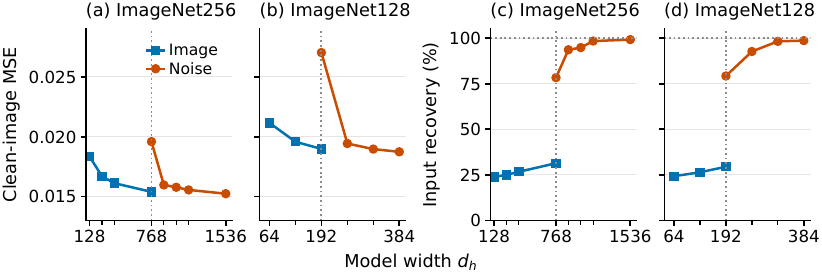}
\caption{Paired denoising and affine input recovery versus model width at epoch $200$ (EMA1) and $t=0.5$. (a,b) Clean-image MSE from the trained denoising head on ImageNet256 and ImageNet128. (c,d) Affine noisy-input recovery from the final-head input at the same resolutions. Blue squares and orange circles mark image and noise prediction. Vertical dotted lines mark $d_h=d_p$; horizontal dotted lines mark $100\%$ input recovery.}
\label{fig:width-probes}
\end{figure}

\subsubsection{Denoising losses across widths and noise levels}
\label{app:width-native-grid}
The checkpoint, image, corruption, and precision protocol is shared with Appendix~\ref{app:exp-native-grid}.
Tables~\ref{tab:width-loss-grid-256} and~\ref{tab:width-loss-grid-128} retain absolute clean-image MSE for all widths and both prediction targets at the two resolutions.
The relative gain in Figures~\ref{fig:width-performance}b and~\ref{fig:width-performance-i128}b is a ratio of sample means, with total clean-image MSE, including Bayes risk, in the denominator.
For its uncertainty intervals, we jointly resample the $2000$ paired image errors and recompute numerator and denominator in each of $1000$ bootstrap resamples, using seed $231023$.
The reference-width gain and interval are identically zero; the $2.5$th and $97.5$th percentiles give the plotted $95\%$ intervals, conditional on the trained checkpoints.
These within-target width comparisons do not imply an ordering of trained image and noise predictors; the absolute-loss grids retain that distinction.

\begin{table}[H]
\centering
\caption{ImageNet256: per-coordinate native clean-image MSE at epoch $200$ (EMA1, $2000$ validation images, no guidance; $\downarrow$). Patch dimension is $d_p=768$. Columns give $t$; the same images and corruptions are used for all widths.}
\label{tab:width-loss-grid-256}
\small
\setlength{\tabcolsep}{3pt}
\begin{tabular}{lrrrrrrrr}
\toprule
Prediction & $d_h$ & .05 & .1 & .2 & .5 & .8 & .9 & .95\\
\midrule
Noise & 768 & 0.54859 & 0.16399 & 0.062464 & 0.019595 & 0.0057381 & 0.0025318 & 0.0011617\\
Noise & 896 & 0.1222 & 0.074817 & 0.043629 & 0.015972 & 0.0049459 & 0.0022827 & 0.0010471\\
Noise & 1024 & 0.12351 & 0.07479 & 0.043377 & 0.015776 & 0.0048654 & 0.0022625 & 0.0010483\\
Noise & 1152 & 0.12228 & 0.074653 & 0.043191 & 0.015557 & 0.0047105 & 0.0021868 & 0.0010165\\
Noise & 1536 & 0.11825 & 0.073 & 0.042467 & 0.015237 & 0.004552 & 0.0020918 & 0.00099022\\
Image & 128 & 0.11259 & 0.073589 & 0.044556 & 0.018345 & 0.0094614 & 0.0080807 & 0.0077457\\
Image & 256 & 0.11145 & 0.07241 & 0.043287 & 0.016699 & 0.0069612 & 0.0051806 & 0.0046526\\
Image & 384 & 0.11087 & 0.071777 & 0.042721 & 0.016122 & 0.006083 & 0.0041282 & 0.003496\\
Image & 768 & 0.11021 & 0.070968 & 0.041957 & 0.015395 & 0.0050643 & 0.0028932 & 0.0021718\\
\bottomrule
\end{tabular}
\end{table}

\begin{table}[H]
\centering
\caption{ImageNet128: per-coordinate native clean-image MSE at epoch $200$ (EMA1, $2000$ validation images, no guidance; $\downarrow$). Patch dimension is $d_p=192$. Columns give $t$; the same images and corruptions are used for all widths.}
\label{tab:width-loss-grid-128}
\small
\setlength{\tabcolsep}{3pt}
\begin{tabular}{lrrrrrrrr}
\toprule
Prediction & $d_h$ & .05 & .1 & .2 & .5 & .8 & .9 & .95\\
\midrule
Noise & 192 & 1.957 & 0.48142 & 0.12913 & 0.027037 & 0.0069793 & 0.0029279 & 0.0012212\\
Noise & 256 & 0.15822 & 0.10029 & 0.058144 & 0.01944 & 0.0055573 & 0.0025714 & 0.0011851\\
Noise & 320 & 0.15593 & 0.099435 & 0.057605 & 0.018981 & 0.0053131 & 0.0024401 & 0.0011328\\
Noise & 384 & 0.15521 & 0.098938 & 0.05719 & 0.01875 & 0.0051801 & 0.0023641 & 0.001104\\
Image & 64 & 0.14759 & 0.098758 & 0.058917 & 0.021162 & 0.0086825 & 0.0068084 & 0.0063907\\
Image & 128 & 0.14647 & 0.097371 & 0.057447 & 0.019596 & 0.0065516 & 0.0042498 & 0.0035854\\
Image & 192 & 0.14597 & 0.096735 & 0.056728 & 0.018996 & 0.0058493 & 0.0034035 & 0.0026305\\
\bottomrule
\end{tabular}
\end{table}

\subsubsection{ADM residual spectra and reference checks}
\label{app:adm-residual-spectrum}
\textbf{Denoisers and evaluation data.}
We use the official class-conditional ImageNet128 and ImageNet256 ADM U-Nets of \citet{dhariwal2021adm}, with checkpoints \texttt{128x128\_diffusion.pt} and \texttt{256x256\_diffusion.pt} from the released implementation.\footnote{\url{https://github.com/openai/guided-diffusion}}
At each resolution, the same $10000$ validation image identities from Appendix~\ref{app:exp-native-grid} receive native-resolution deterministic center crops normalized to $[-1,1]$: the first $6000$ fit the affine baseline and the final $4000$ estimate the spectrum, giving six fitting and four evaluation images per class.
All measurements use true labels, no guidance, no clipping, and the seven times in Appendix~\ref{app:exp-native-grid}.
Image selection and ordering follow Appendix~\ref{app:exp-native-grid}, and each corruption uses seed $23102026$ plus the image's position; each image shares its Gaussian draw across times.
ADM receives $Y_t/\sqrt{1+\sigma_t^2}$ in its variance-preserving convention.
We interpolate its time embedding linearly in log effective noise between the original $1000$ linear-schedule steps, then apply \eqref{eq:noise-induced-clean}; the variance head is unused. Timestep interpolation is an approximation.
Inference and conversion use FP32; covariance estimates and eigendecompositions use FP64.

\textbf{Fitted baseline and residual.}
We analyze one fixed RGB patch at a time: $8\times8$ at resolution $128$ and $16\times16$ at resolution $256$.
Let $\widehat\mu_m,\widehat\Sigma_m$ be its clean mean and covariance from the $6000$ fitting images, with covariance divisor equal to the sample count.
The formula in Appendix~\ref{app:affine-baseline} gives
\begin{align}
 \widehat\ell_t(y)&=\widehat\mu_m+\widehat\Sigma_m(\widehat\Sigma_m+\sigma_t^2I_{d_p})^{-1}(y-\widehat\mu_m),\notag\\
 \widehat R_t&=[\widehat X_1^{(\mathrm{ADM})}(Y_t)]_m-\widehat\ell_t(Y_{t,m}).
 \label{eq:adm-residual-estimate}
\end{align}
Clean moments avoid fitting incidental clean--noise correlations from finitely many corruptions.
The baseline pools classes and uses only the noisy patch; ADM also observes the whole image and class label. Its residual can therefore reflect nonlinear local dependence, information from other patches, and class information.

\textbf{Covariance directions and spatial variation.}
We subtract the empirical residual mean on the $4000$ evaluation images and form the covariance with divisor $4000$.
A residual direction is a unit covariance eigenvector in the patch-output space $\mbb R^{d_p}$.
Magnitude and the $95\%$ direction count use the definitions in Section~\ref{sec:exp-hidden-width}.
Figure~\ref{fig:width-performance}c uses patch $(7,7)$ on the zero-based $16\times16$ patch grid; shading spans separate estimates at $(7,7),(3,3),(3,12),(12,3),(12,12)$.
Positions are analyzed separately, so the band represents spatial variation rather than a confidence interval.
Figure~\ref{fig:width-performance-i128}c gives the corresponding native-resolution ImageNet128 spectrum.

\textbf{Finite-sample and spatial checks.}
For ImageNet256, the five positions give $95\%$ direction counts of $5$--$6$ at $t=0.1$ and $388$--$416$ at $t=0.9$.
Using $3000$ instead of $6000$ baseline-fitting images gives central-patch counts $6$ and $419$.
Learning directions on half the evaluation images and measuring captured variance on the other, reversing roles, gives $6$ and $500$.
The qualitative contrast persists, while exact low-noise counts depend on finite-sample estimation.
The $95\%$ count describes relative concentration; an absolute approximation tolerance also depends on residual magnitude.

\textbf{Denoising quality of the reference.}
Table~\ref{tab:adm-denoising-quality} compares ADM with the epoch-$200$ EMA1 full-input JiT noise and image predictors, using widths $384$ and $192$ at resolution $128$, and $1536$ and $768$ at resolution $256$.
Every row within a resolution uses the same $4000$ validation images and corruptions with FP32 inference and conversion.
This reference-quality comparison is distinct from the $2000$-image, BF16-backbone native-loss grid.
The reference predictions and fitted residual spectrum remain estimates of the population quantities.

\begin{table}[H]
\centering
\caption{Denoising quality of the resolution-matched ADM references: full-image per-coordinate clean-image MSE on $4000$ common validation images within each resolution, without guidance. Columns give $t$; bold marks the lowest measured MSE at each time. All rows use FP32 inference and conversion.}
\label{tab:adm-denoising-quality}
\small
\setlength{\tabcolsep}{3pt}
\begin{tabular}{lrrrrrrr}
\toprule
Denoiser & .05 & .1 & .2 & .5 & .8 & .9 & .95\\
\midrule
\multicolumn{8}{l}{ImageNet256}\\
JiT noise ($d_h=1536$) & 0.111033 & 0.071357 & 0.042082 & 0.015287 & 0.004591 & 0.002104 & 0.000993\\
JiT image ($d_h=768$) & \textbf{0.110087} & 0.070997 & 0.041916 & 0.015472 & 0.005118 & 0.002918 & 0.002197\\
ADM U-Net & 0.110465 & \textbf{0.070728} & \textbf{0.040943} & \textbf{0.014005} & \textbf{0.003711} & \textbf{0.001485} & \textbf{0.000594}\\
\midrule
\multicolumn{8}{l}{ImageNet128}\\
JiT noise ($d_h=384$) & 0.145839 & 0.096349 & 0.056384 & 0.018685 & 0.005189 & 0.002365 & 0.001104\\
JiT image ($d_h=192$) & 0.145676 & 0.096343 & 0.056408 & 0.018963 & 0.005866 & 0.003411 & 0.002636\\
ADM U-Net & \textbf{0.145347} & \textbf{0.095172} & \textbf{0.054216} & \textbf{0.016763} & \textbf{0.004073} & \textbf{0.001563} & \textbf{0.000596}\\
\bottomrule
\end{tabular}
\end{table}

\section{Related literature}
\label{app:related-literature}
\noindent\textbf{Prediction targets and data geometry.}
JiT motivates image prediction through the manifold hypothesis and demonstrates its advantage in coarse-patch Transformers, including models with narrow input embeddings \citep{li2025jit}.
\citet{jin2026dimensionality} connect data geometry to the preferred mixture of clean data and noise in the prediction target through an analysis of a linear predictor shared across time, and propose $k$-Diff to learn this mixture jointly with the network.
These works establish the importance of data geometry when choosing a prediction target.

A further question is whether the architecture preserves the input information needed to predict that target accurately.
Even directions containing little clean-image information can carry noise that must be canceled when forming a clean-image estimate from a noise predictor.
Our analysis quantifies the denoising error caused by losing this information: Section~\ref{sec:input} establishes a rank-dependent error floor even with unrestricted downstream computation, requiring only Gaussian corruption and a finite second moment for the clean data.
Once input compression is removed, Section~\ref{sec:joint-width} characterizes the additional width needed to represent corrections beyond linear denoising, assuming affine recovery of the noisy patch.

\noindent\textbf{Architecture and prediction-target differences.}
\citet{gagneux2026training} place prediction parameterizations in a common denoising formulation, separate the induced denoiser class from loss weighting, and show that target preference depends strongly on architecture: small-patch Transformers favor velocity prediction, while larger patches favor image prediction in their tested settings.
Their controlled-manifold experiments also show that intrinsic dimension interacts with the architecture.
These experiments establish that architectural locality can change which prediction target performs best.
Complementing this experimental evidence, RAE derives a sampled-velocity loss floor for a network with linear input and output maps around a nonlinear body \citep[Theorem~1 and Appendix~B.1]{zheng2025rae}.
Its proof uses the restriction of every prediction to the column space of the final output map.
This identifies a limitation on the directions the output can represent.

These results leave open how information discarded at the input translates into unavoidable denoising error, and how much representation width is needed to retain the noisy input while supporting accurate denoising.
Within the common denoising viewpoint, Theorem~\ref{prop:input} characterizes information unavailable from the input while allowing unrestricted decoder outputs, and Corollary~\ref{cor:amplification} translates the resulting noise-recovery floor into denoising error.
Theorem~\ref{thm:joint-width} and Corollary~\ref{cor:joint-tolerance} then characterize the width required for a chosen denoising accuracy, assuming exact affine recovery of the noisy patch.

\noindent\textbf{Linear denoising and affine input recovery.}
\citet{wang2024gaussianscore} study Gaussian score approximations based on the data mean and covariance, and provide theoretical and empirical evidence that these linear approximations describe learned scores well at moderate to high noise.
\citet{li2024hiddengaussian} fit linear models to trained denoisers and find that, in the generalization regime, these approximations closely resemble the optimal denoisers for a Gaussian distribution with the empirical data mean and covariance.
They also report near-identity behavior at low noise, distinguishing this behavior from evidence of learned Gaussian structure \citep[Appendix~D]{li2024hiddengaussian}.
These observations motivate the representation strategy in Section~\ref{sec:joint-width}: keeping the noisy patch affinely accessible provides the features needed for linear denoising, while additional representation coordinates can encode the remaining correction.
We formalize this strategy in Assumption~\ref{ass:patch-recovery}, which requires exact recovery of the noisy patch through an affine readout of the model representation.
Such a representation supports any affine denoiser of that patch, including the best affine estimate used as our linear-denoising baseline.
This connects the observed linear structure of learned denoisers to our analysis of the additional width needed for corrections beyond that baseline.

\noindent\textbf{Preconditioning and analytic input paths.}
EDM combines a scaled copy of the noisy input with a learned output using noise-dependent coefficients \citep{karras2022edm}; consistency models use a related construction to enforce their boundary condition \citep{song2023consistency}.
Related approaches use analytic score components for subspace data \citep{chen2023score} or project the noise target while recovering the full velocity analytically \citep{chen2026asymflow}.
JiT reports that output preconditioning improves over direct noise and velocity prediction but remains worse than image prediction, and links the improvement to its high-noise approach to image prediction \citep[Appendix~B.2 and Table~10]{li2025jit}.
Section~\ref{sec:conditioning-theory} quantifies this pattern through the input-rank-dependent contribution of discarded noise, distinguishing EDM's high-noise attenuation from exact cancellation at every noise level.
Theorem~\ref{thm:edm-target-convergence} also formalizes the raw target's mean-square convergence to a scaled clean image.
Data preconditioning improves covariance conditioning and optimization in Gaussian and Gaussian-mixture settings \citep{ahamed2026preconditioned}.
This complements our output-preconditioning analysis of the best attainable risk under architectural information constraints.

\end{document}